\documentclass{article}

\usepackage[final,nonatbib]{style/neurips_2021}
\usepackage{chngcntr}
\usepackage{appendix}

\usepackage{amsmath,amsfonts,amssymb,amsthm}
\newtheorem{proposition}{Proposition}
\newtheorem{definition}{Definition}
\newtheorem{lemma}{Lemma}

\newtheorem*{proposition*}{Proposition}
\newtheorem*{lemma*}{Lemma}

\usepackage[style=numeric-comp,sorting=none,maxbibnames=20,maxcitenames=1,doi=false,isbn=false,url=false,eprint=false]{biblatex}
\AtEveryBibitem{\clearfield{month}}
\AtEveryCitekey{\clearfield{month}}
\AtEveryBibitem{\clearlist{language}}
\renewbibmacro{in:}{}
\usepackage{graphicx}

\DeclareRobustCommand{\mb}[1]{\boldsymbol{#1}}

\DeclareMathOperator*{\argmax}{argmax}
\DeclareMathOperator*{\argmin}{argmin}
\DeclareMathOperator*{\Tr}{Tr}

\renewcommand{\mid}{~\vert~}

\newcommand{\mbb}{\mb{b}}

\newcommand{\mbx}{\mb{x}}
\newcommand{\mby}{\mb{y}}
\newcommand{\mbz}{\mb{z}}

\newcommand{\mbA}{\mb{A}}
\newcommand{\mbB}{\mb{B}}
\newcommand{\mbC}{\mb{C}}
\newcommand{\mbD}{\mb{D}}

\newcommand{\mbH}{\mb{H}}
\newcommand{\mbI}{\mb{I}}

\newcommand{\mbK}{\mb{K}}

\newcommand{\mbM}{\mb{M}}

\newcommand{\mbQ}{\mb{Q}}
\newcommand{\mbR}{\mb{R}}
\newcommand{\mbS}{\mb{S}}
\newcommand{\mbT}{\mb{T}}
\newcommand{\mbU}{\mb{U}}
\newcommand{\mbV}{\mb{V}}
\newcommand{\mbW}{\mb{W}}
\newcommand{\mbX}{\mb{X}}
\newcommand{\mbY}{\mb{Y}}
\newcommand{\mbZ}{\mb{Z}}

\newcommand{\mbPi}{\mb{\Pi}}

\newcommand{\bbR}{\mathbb{R}}
\newcommand{\bbS}{\mathbb{S}}

\newcommand{\cG}{\mathcal{G}}
\newcommand{\cH}{\mathcal{H}}

\newcommand{\cO}{\mathcal{O}}
\newcommand{\cP}{\mathcal{P}}
\newcommand{\cQ}{\mathcal{Q}}

\newcommand{\cS}{\mathcal{S}}

\newcommand{\cV}{\mathcal{V}}

\newcommand{\reals}{\mathbb{R}}
\newcommand{\expect}{\mathbb{E}}

\newcommand{\textapprox}{\raisebox{0.5ex}{\texttildelow}}

\usepackage[utf8]{inputenc} 
\usepackage[T1]{fontenc}    
\usepackage[nameinlink,capitalise]{cleveref}
\usepackage{url}            
\usepackage{booktabs}       
\usepackage{amsfonts}       
\usepackage{nicefrac}       
\usepackage{microtype}      

\usepackage[font=small]{caption}
\title{\hspace{-.075em}Generalized Shape Metrics on Neural Representations}

\author{%
  Alex H. Williams \\
  Statistics Department\\
  Stanford University \\
  \texttt{ahwillia@stanford.edu}
  \And
  Erin Kunz \\
  Electrical Engineering Department \\
  Stanford University \\
  \texttt{ekunz@stanford.edu}
  \AND
  Simon Kornblith \\
  Google Research, Toronto \\
  \texttt{skornblith@google.com}
  \And
  Scott W. Linderman \\
  Statistics Department\\
  Stanford University \\
  \texttt{scott.linderman@stanford.edu}
}

\begin{document}

\maketitle

\begin{abstract}
\noindent
Understanding the operation of biological and artificial networks remains a difficult and important challenge.
To identify general principles, researchers are increasingly interested in surveying large collections of networks that are trained on, or biologically adapted to, similar tasks.
A standardized set of analysis tools is now needed to identify how network-level covariates---such as architecture, anatomical brain region, and model organism---impact neural representations (hidden layer activations).
Here, we provide a rigorous foundation for these analyses by defining a broad family of metric spaces that quantify representational dissimilarity.
Using this framework we modify existing representational similarity measures based on canonical correlation analysis to satisfy the triangle inequality, formulate a novel metric that respects the inductive biases in convolutional layers, and identify approximate Euclidean embeddings that enable network representations to be incorporated into essentially any off-the-shelf machine learning method.
We demonstrate these methods on large-scale datasets from biology (Allen Institute Brain Observatory) and deep learning (NAS-Bench-101).
In doing so, we identify relationships between neural representations that are interpretable in terms of anatomical features and model performance.
\end{abstract}

\section{Introduction}

The extent to which different deep networks or neurobiological systems use equivalent representations in support of similar task demands is a topic of persistent interest in machine learning and neuroscience~\cite{Barrett2019,Kriegeskorte2021,Roeder2021}.
Several methods including linear regression~\cite{Yamins2014,Cadena2019}, canonical correlation analysis~(CCA;~\cite{Raghu2017,Morcos2018}), representational similarity analysis~(RSA;~\cite{Kriegeskorte2008}), and centered kernel alignment~(CKA;~\cite{Kornblith2019}) have been used to quantify the similarity of hidden layer activation patterns.
These measures are often interpreted on an ordinal scale and are employed to compare a limited number of networks---e.g., they can indicate whether networks $A$ and $B$ are more or less similar than networks $A$ and $C$.
While these comparisons have yielded many insights~\cite{Kriegeskorte2008,Raghu2017,Morcos2018,Kornblith2019,Maheswaranathan2019,Nguyen2020,Yamins2014,Cadena2019,Shi2019}, the underlying methodologies have not been extended to systematic analyses spanning thousands of networks.

To unify existing approaches and enable more sophisticated analyses, we draw on ideas from \textit{statistical shape analysis}~\cite{Small1996,Kendall1999,dryden2016} to develop dissimilarity measures that are proper metrics---i.e., measures that are symmetric and respect the triangle inequality.
This enables several off-the-shelf methods with theoretical guarantees for classification (e.g. k-nearest neighbors,~\cite{Yianilos1993}) and clustering (e.g. hierarchical clustering~\cite{Dasgupta2005}).
Existing similarity measures can violate the triangle inequality, which complicates these downstream analyses~\cite{Baraty2011,Wang2015,Chang2016}.
However, we show that existing dissimilarity measures can often be modified to satisfy the triangle inequality and viewed as special cases of the framework we outline.
We also describe novel metrics within this broader family that are specialized to convolutional layers and have appealing properties for analyzing artificial networks.

Moreover, we show empirically that these metric spaces on neural representations can be embedded with low distortion into Euclidean spaces, enabling an even broader variety of previously unconsidered supervised and unsupervised analyses.
For example, we can use neural representations as the inputs to linear or nonlinear regression models.
We demonstrate this approach on neural representations in mouse visual cortex (Allen Brain Observatory;~\cite{Siegle2021}) in order to predict each brain region's anatomical hierarchy from its pattern of visual responses---i.e., predicting a feature of brain structure from function.
We demonstrate a similar approach to analyze hidden layer representations in a database of 432K deep artificial networks (NAS-Bench-101;~\cite{Ying19}) and find a surprising degree of correlation between early and deep layer representations.

Overall, we provide a theoretical grounding which explains why existing representational similarity measures are useful: they are often close to metric spaces, and can be modified to fulfill metric space axioms precisely.
Further, we draw new conceptual connections between analyses of neural representations and established research areas~\cite{Cohen2016,dryden2016}, utilize these insights to propose novel metrics, and demonstrate a general-purpose machine learning workflow that scales to datasets with thousands of networks.
\begin{figure}
\centering
\includegraphics[width=\linewidth]{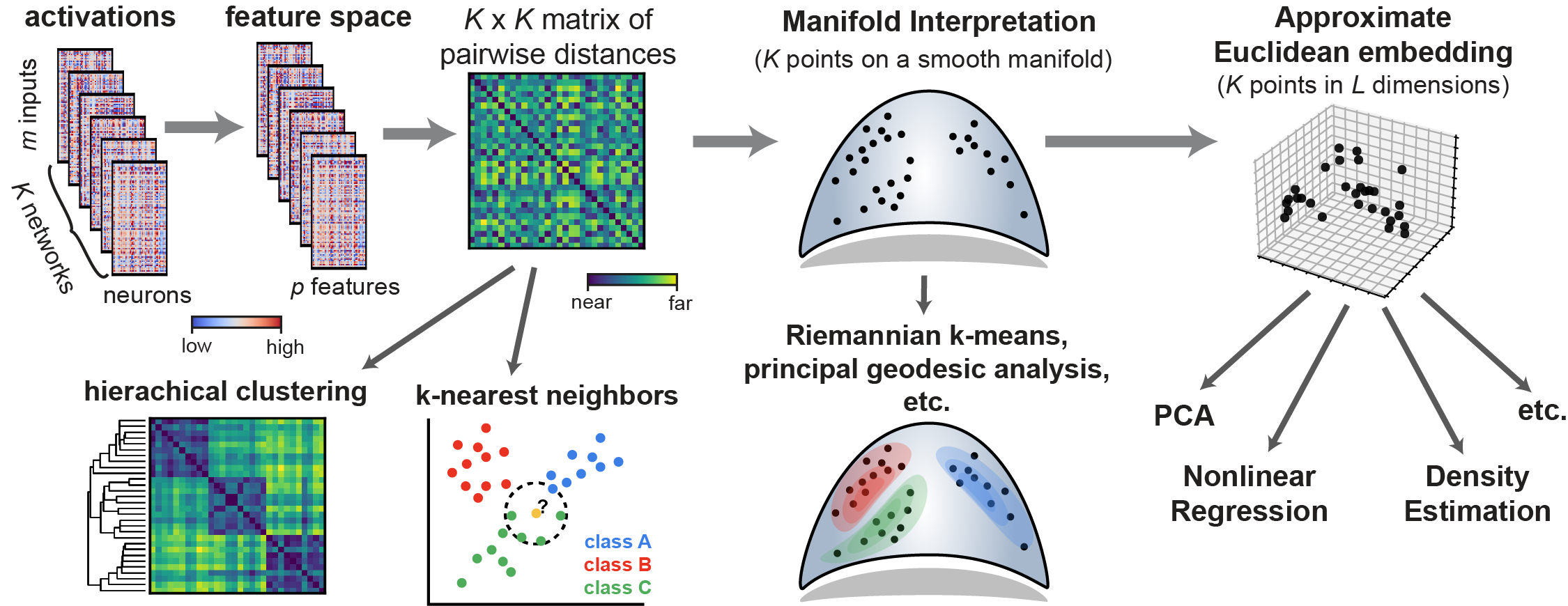}
\caption{
Machine learning workflows enabled by generalized shape metrics.
}
\label{fig1}
\end{figure}

\section{Methods}

This section outlines several workflows (\cref{fig1}) to analyze representations across large collections of networks.
After briefly summarizing prior approaches (sec. \ref{subsec:prior-work}), we cover background material on metric spaces and discuss their theoretical advantages over existing dissimilarity measures (sec. \ref{subsec:metrics-background}).
We then present a class of metrics that capture these advantages (sec. \ref{subsec:generalized-shape-metrics}) and cover a special case that is suited to convolutional layers (sec. \ref{subsec:convolutional-shape-metrics}).
We then demonstrate the practical advantages of these methods in \cref{sec:applications}, and demonstrate empirically that Euclidean feature spaces can approximate the metric structure of neural representations, enabling a broad set of novel analyses.


\subsection{Prior work and problem setup}
\label{subsec:prior-work}

Neural network representations are often summarized over a set of $m$ reference inputs (e.g. test set images).
Let $\mbX_i \in \reals^{m \times n_i}$ and $\mbX_j \in \reals^{m \times n_j}$ denote the responses of two networks (with $n_i$ and $n_j$ neurons, respectively) to a collection of these inputs.
Quantifying the similarity between $\mbX_i$ and $\mbX_j$ is complicated by the fact that, while the $m$ inputs are the same, there is no direct correspondence between the neurons.
Even if $n_i = n_j$, the typical Frobenius inner product, $\langle \mbX_i, \mbX_j \rangle = \Tr [ \mbX_i^\top \mbX_j ]$, and metric, $\Vert \mbX_i - \mbX_j \Vert = \langle \mbX_i - \mbX_j, \mbX_i - \mbX_j \rangle^{1/2}$, fail to capture the desired notion of dissimilarity.
For instance, let $\mbPi$ denote some $n \times n$ permutation matrix and let $\mbX_i = \mbX_j \mbPi$.
Intuitively, we should consider $\mbX_i$ and $\mbX_j$ to be identical in this case since the ordering of neurons is arbitrary.
Yet, clearly $\Vert \mbX_i - \mbX_j \Vert \neq 0$, except in very special cases.

One way to address this problem is to linearly regress over the neurons to predict $\mbX_i$ from $\mbX_j$.
Then, one can use the coefficient of determination~($R^2$) as a measure of similarity~\cite{Yamins2014, Cadena2019}.
However, this similarity score is asymmetric---if one instead treats $\mbX_j$ as the dependent variable that is predicted from $\mbX_i$, this will result in a different $R^2$.
Canonical correlation analysis (CCA; \cite{Raghu2017,Morcos2018}) and linear centered kernel alignment (linear CKA; \cite{Kornblith2019,Cortes2012}) also search for linear correspondences between neurons, but have the advantage of producing symmetric scores.
Representational similarity analysis (RSA;~\cite{Kriegeskorte2008}) is yet another approach, which first computes an $m \times m$ matrix holding the dissimilarities between all pairs of representations for each network.
These \textit{representational dissimilarity matrices} (RDMs), are very similar to the $m \times m$ kernel matrices computed and compared by CKA.
RSA traditionally quantifies the similarity between two neural networks by computing Spearman's rank correlation between their RDMs.
A very recent paper by \textcite{Shahbazi2021}, which was published while this manuscript was undergoing review, proposes to use the Riemannian metric between positive definite matrices instead of Spearman correlation.
Similar to our results, this establishes a metric space that can be used to compare neural representations.
Here, we leverage metric structure over \textit{shape spaces} \cite{Small1996,Kendall1999,dryden2016} instead of positive definite matrices, leading to complementary insights.

In summary, there are a diversity of methods that one can use to compare neural representations.
Without a unifying theoretical framework it is unclear how to choose among them, use their outputs for downstream tasks, or generalize them to new domains.

\subsection{Feature space mapping, metrics, and equivalence relations}
\label{subsec:metrics-background}

Our first contribution will be to establish formal notions of distance (metrics) between neural representations.
To accommodate the common scenario when the number of neurons varies across networks (i.e. when $n_i \neq n_j$), we first map the representations into a common feature space.
For each set of representations, $\mbX_i$, we suppose there is a mapping into a $p$-dimensional feature space, $\mbX_i \mapsto \mbX_i^\phi$, where $\mbX_i^\phi \in \reals^{m \times p}$.
In the special case where all networks have equal size, $n_1 = n_2 = \hdots = n$, we can express the feature mapping as a single function $\phi : \reals^{m \times n} \mapsto \reals^{m \times p}$, so that $\mbX_i^\phi = \phi(\mbX_i)$.
When networks have dissimilar sizes, we can map the representations into a common dimension using, for example, PCA \cite{Raghu2017}.

Next, we seek to establish \textit{metrics} within the feature space, which are distance functions that satisfy:
\begin{align}
\text{Equivalence:}  \quad & d(\mbX_i^\phi, \mbX_j^\phi) = 0 ~\iff~ \mbX_i^\phi \sim \mbX_j^\phi \label{eq:metric-equivalence}
\\
\text{Symmetry:} \quad & d(\mbX_i^\phi, \mbX_j^\phi) = d(\mbX_j^\phi, \mbX_i^\phi)
\label{eq:metric-symmetry}
\\
\text{Triangle Inequality:} \quad & d(\mbX_i^\phi, \mbX_j^\phi) \leq d(\mbX_i^\phi, \mbX_k^\phi) + d(\mbX_k^\phi, \mbX_j^\phi)
\label{eq:metric-triangle-inequality}
\end{align}
for all $\mbX_i^\phi$, $\mbX_j^\phi$, and $\mbX_k^\phi$ in the feature space.
The symbol `$\sim$' denotes an \textit{equivalence relation} between two elements.
That is, the expression $\mbX_i^\phi \sim \mbX_j^\phi$ means that ``$\mbX_i^\phi$ is equivalent to $\mbX_j^\phi$.''
Formally, distance functions satisfying \cref{eq:metric-equivalence,eq:metric-symmetry,eq:metric-triangle-inequality} define a metric over a quotient space defined by the equivalence relation and a pseudometric over $\reals^{m \times p}$ (see Supplement A).
Intuitively, by specifying different equivalence relations we can account for symmetries in network representations, such as permutations over arbitrarily labeled neurons (other options are discussed below in sec. \ref{subsec:generalized-shape-metrics}).

Metrics quantify dissimilarity in a way that agrees with our intuitive notion of distance.
For example,~\cref{eq:metric-symmetry} ensures that the distance from $\mbX_i^\phi$ to $\mbX_j^\phi$ is the same as the distance from $\mbX_j^\phi$ to~$\mbX_i^\phi$.
Linear regression is an approach that violates this condition: the similarity measured by $R^2$ depends on which network is treated as the dependent variable.

Further,~\cref{eq:metric-triangle-inequality} ensures that distances are self-consistent in the sense that if two elements ($\mbX_i^\phi$ and $\mbX_j^\phi$) are both close to a third ($\mbX_k^\phi$), then they are necessarily close to each other.
Many machine learning models and algorithms rely on this triangle inequality condition.
For example, in clustering, it ensures that if $\mbX_i^\phi$ and $\mbX_j^\phi$ are put into the same cluster as $\mbX_k^\phi$, then $\mbX_i^\phi$ and $\mbX_j^\phi$ cannot be too far apart, thus implying that they too can be clustered together.
Intuitively, this establishes an appealing transitive relation for clustering, which can be violated when the triangle inequality fails to hold.
Existing measures based on CCA, RSA, and CKA, are symmetric, but do not satisfy the triangle inequality.
By modifying these approaches to satisfy the triangle inequality, we avoid potential pitfalls and can leverage theoretical guarantees on learning in proper metric spaces \cite{Yianilos1993,Dasgupta2005,Baraty2011,Wang2015,Chang2016}.

\subsection{Generalized shape metrics and group invariance}
\label{subsec:generalized-shape-metrics}

In this section, we outline a new framework to quantify representational dissimilarity, which leverages a well-developed mathematical literature on \textit{shape spaces} \cite{dryden2016,Kendall1999,Small1996}.
The key idea is to treat ${\mbX_i^\phi \sim \mbX_j^\phi}$ if and only if there exists a linear transformation $\mbT$ within a set of allowable transformations $\cG$, such that $\mbX_i^\phi = \mbX_j^\phi \mbT$.
Although $\cG$ only contains linear functions, nonlinear alignments between the raw representations can be achieved when the feature mappings $\mbX_i \mapsto \mbX^\phi_i$ are chosen to be nonlinear.
Much of shape analysis literature focuses on the special case where $p = n$ and $\cG$ is the special orthogonal group $\cS\cO(n) = \{\mbR \in \reals^{n \times n} \mid \mbR^\top \mbR = \mbI , \det(\mbR) = 1 \}$, meaning that $\mbX_i^\phi$ and $\mbX_j^\phi$ are equivalent if there is a $n$-dimensional rotation (without reflection) that relates them.
Standard shape analysis further considers each $\mbX_i^\phi$ to be a mean-centered ($(\mbX_i^\phi)^\top \boldsymbol{1} = \boldsymbol{0}$) and normalized ($\Vert \mbX_i^\phi \Vert = 1$) version of the raw landmark locations held in $\mbX_i \in \reals^{m \times n}$ (an assumption that we will relax).
That is, the feature map $\phi : \bbR^{m \times n} \mapsto \bbS^{m \times n}$ transforms the raw landmarks onto the hypersphere, denoted $\bbS^{m \times n}$, of $m \times n$ matrices with unit Frobenius norm.
In this context, $\mbX_i^\phi \in \bbS^{m \times n}$ is called a ``pre-shape.''
By removing rotations from a pre-shape, $[\mbX^\phi_i] = \{\mbS \in \bbS^{m \times n} \mid \mbS \sim \mbX^\phi_i \}$ for pre-shape $\mbX_i^\phi$, we recover its ``shape.''

\begin{figure}
\centering
\includegraphics[width=\linewidth]{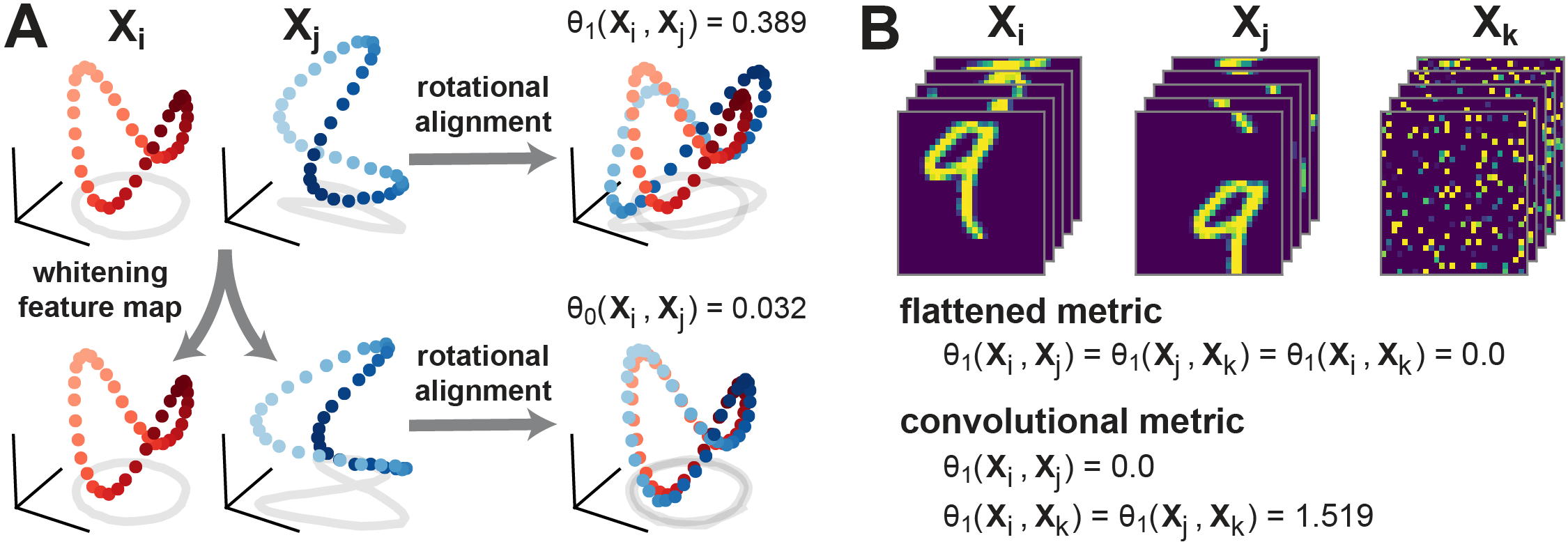}
\caption{
\textit{(A)} Schematic illustration of metrics with rotational invariance (top), and linear invariance (bottom).
Red and blue dots represent a pair of network representations $\mbX_i$ and $\mbX_j$, which correspond to $m$ points in $n$-dimensional space.
\textit{(B)} Demonstration of convolutional metric on toy data.
Flattened metrics (e.g. \cite{Raghu2017,Kornblith2019}) that ignore convolutional layer structure treat permuted images ($\mbX_k$, right) as equivalent to images with coherent spatial structure ($\mbX_i$ and $\mbX_j$, left and middle).
A convolutional metric, \cref{eq:circ-convolution-equivalence}, distinguishes between these cases while still treating $\mbX_i$ and $\mbX_j$ as equivalent (obeying translation invariance).
}
\label{fig2}
\end{figure}

To quantify dissimilarity in neural representations, we generalize this notion of shape to include other feature mappings and alignments.
The minimal distance within the feature space, after optimizing over alignments, defines a metric under suitable conditions (\cref{fig2}A).
This results in a broad variety of \textit{generalized shape metrics} (see also, ch. 18 of \cite{dryden2016}), which fall into two categories as formalized by the pair of propositions below.
Proofs are provided in Supplement B.

\begin{proposition}
\label{proposition:euc-shape-metric}
Let $\mbX_i^\phi \in \reals^{m \times p}$, and let $\cG$ be a group of linear isometries on $\reals^{m \times p}$. Then,
\begin{equation}
d(\mbX_i^\phi, \mbX_j^\phi) = \min_{\mbT \in \cG} ~ \Vert \mbX^\phi_i - \mbX^\phi_j \mbT \Vert
\end{equation}
defines a metric, where $\mbX_i^\phi \sim \mbX_j^\phi$ if and only if there is a $\mbT \in \cG$ such that $\mbX^\phi_i = \mbX^\phi_j \mbT$.
\end{proposition}

\begin{proposition}
\label{proposition:ang-shape-metric}
Let $\mbX_i^\phi \in \bbS^{m \times p}$, and let $\cG$ be a group of linear isometries on $\bbS^{m \times p}$.
Then,
\begin{equation}
\theta(\mbX_i^\phi, \mbX_j^\phi) = \min_{\mbT \in \cG} ~ \arccos \big \langle \mbX^\phi_i, \mbX^\phi_j \mbT \big \rangle
\end{equation}
defines a metric, where $\mbX_i^\phi \sim \mbX_j^\phi$ if and only if there is a $\mbT \in \cG$ such that $\mbX^\phi_i = \mbX^\phi_j \mbT$.
\end{proposition}

Two key conditions appear in these propositions.
First, $\cG$ must be a \textit{group} of functions.
This means $\cG$ is a set that contains the identity function, is closed under composition ($\mbT_1 \mbT_2 \in \cG$ for any $\mbT_1 \in \cG$ and $\mbT_2 \in \cG$), and whose elements are invertible by other members of the set (if $\mbT \in \cG$ then $\mbT^{-1} \in \cG$).
Second, every $\mbT \in \cG$ must be an \textit{isometry}, meaning that $\Vert \mbX^\phi_i - \mbX^\phi_j \Vert = \Vert \mbX^\phi_i \mbT - \mbX^\phi_j \mbT \Vert$ for all $\mbT \in \cG$ and all elements of the feature space.
On $\reals^{m \times p}$ and $\bbS^{m \times p}$, all linear isometries are orthogonal transformations.
Further, the set of orthogonal transformations, $\cO(p) = \{\mbQ \in \reals^{p \times p} : \mbQ^\top \mbQ = \mbI \}$, defines a well-known group.
Thus, the condition that $\cG$ is a group of isometries is equivalent to $\cG$ being a subgroup of $\cO(p)$---i.e., a subset of $\cO(p)$ satisfying the group axioms.

Intuitively, by requiring $\cG$ to be a group of functions, we ensure that the alignment procedure is symmetric---i.e. it is equivalent to transform $\mbX_i^\phi$ to match $\mbX_j^\phi$, or transform the latter to match the former.
Further, by requiring each $\mbT \in \cG$ to be an isometry, we ensure that the underlying metric (Euclidean distance for \cref{proposition:euc-shape-metric}; angular distance for \cref{proposition:ang-shape-metric}) preserves its key properties.

Together, these propositions define a broad class of metrics as we enumerate below.
For simplicity, we assume that $n_i = n_j = n$ in the examples below, with the understanding that a PCA or zero-padding preprocessing step has been performed in the case of dissimilar network sizes.
This enables us to express the metrics as functions of the raw activations, i.e. functions $\reals^{m \times n} \times \reals^{m \times n} \mapsto \reals_+$.

\paragraph{Permutation invariance}
The most stringent notion of representational similarity is to demand that neurons are one-to-one matched across networks.
If we set the feature map to be the identity function, i.e., $\mbX_i^\phi = \mbX_i$ for all $i$, then: 
\begin{equation}
\label{eq:permutation-distance}
d_\cP(\mbX_i, \mbX_j) = \min_{\mbPi \in \cP(n)} \Vert \mbX_i - \mbX_j \mbPi \Vert
\end{equation}
defines a metric by \Cref{proposition:euc-shape-metric} since the set of permutation matrices, $\cP(n)$, is a subgroup of~$\cO(n)$.
To evaluate this metric we must optimize over the set of neuron permutations to align the two networks.
This can be reformulated (see Supplement C) as a fundamental problem in combinatorial optimization known as the linear assignment problem \cite{Burkard2012}.
Exploiting an algorithm due to Jonker and Volgenant \cite{Jonker1987,Crouse2016} we can solve this problem in $O(n^3)$ time.
The overall runtime for evaluating \cref{eq:permutation-distance} is $O(mn^2 + n^3)$, since we must evaluate $\mbX_i^\top \mbX_j$ to formulate the assignment problem.

\paragraph{Rotation invariance}
Let $\mbC = \mbI_m - (1/m) \boldsymbol{1} \boldsymbol{1}^\top$ denote an $m \times m$ \textit{centering matrix}, and consider the feature mapping $\phi_1$ which mean-centers the columns, $\phi_1(\mbX_i) = \mbX_i^{\phi_1} = \mbC \mbX_i$. Then,
\begin{equation}
\label{eq:procrustes-size-and-shape-distance}
d_1(\mbX_i, \mbX_j) = \min_{\mbQ \in \cO} \Vert \mbX_i^{\phi_1} - \mbX_j^{\phi_1} \mbQ \Vert
\end{equation}
defines a metric by \Cref{proposition:euc-shape-metric}, and is equivalent to the \textit{Procrustes size-and-shape distance} with reflections \cite{dryden2016}.
Further, by \Cref{proposition:ang-shape-metric},
\begin{equation}
\label{eq:riemannian-shape-distance}
\theta_1(\mbX_i, \mbX_j) = \min_{\mbQ \in \cO} ~ \arccos \, \frac{\langle \mbX_i^{\phi_1}, \mbX_j^{\phi_1} \mbQ \rangle}{\Vert \mbX_i^{\phi_1} \Vert \Vert \mbX_j^{\phi_1} \Vert}
\end{equation}
defines another metric, and is closely related to the Riemannian distance on Kendall's shape space \cite{dryden2016}.
To evaluate \cref{eq:procrustes-size-and-shape-distance,eq:riemannian-shape-distance}, we must optimize over the set of orthogonal matrices to find the best alignment.
This also maps onto a fundamental optimization problem known as the \textit{orthogonal Procrustes problem} \cite{Schonemann1966,Gower2004}, which can be solved in closed form in $O(n^3)$ time.
As in the permutation-invariant metric described above, the overall runtime is $O(mn^2 + n^3)$. 

\paragraph{Linear invariance}
Consider a partial whitening transformation, parameterized by $0 \leq \alpha \leq 1$:
\begin{equation}
\label{eq:whitening-feature-map}
\mbX^{\phi_\alpha} = \mbC \mbX (\alpha \mbI_n + (1 - \alpha) (\mbX^\top \mbC \mbX)^{-1/2})
\end{equation}
Note that $\mbX^\top \mbC \mbX$ is the empirical covariance matrix of $\mbX$.
Thus, when $\alpha=0$, \cref{eq:whitening-feature-map} corresponds to ZCA whitening \cite{Kessy2018}, which intuitively removes invertible linear transformations from the representations.
Thus, when $\alpha = 0$ the metric outlined below treats $\mbX_i \sim \mbX_j$ if there exists an affine transformation that relates them: $\mbX_i = \mbX_j \mbW + \mbb$ for some $\mbW \in \reals^{n \times n}$ and $\mbb \in \reals^n$.
When $\alpha=1$, \cref{eq:whitening-feature-map} reduces to the mean-centering feature map used above.

Using orthogonal alignments within this feature space leads to a metric that is related to CCA.
First, let $\rho_1 \geq \hdots \geq \rho_n \geq 0$ denote the singular values of $(\mbX_i^{\phi_\alpha})^\top (\mbX_j^{\phi_\alpha}) / \Vert \mbX_i^{\phi_\alpha} \Vert \Vert \mbX_j^{\phi_\alpha} \Vert$.
One can show that
\begin{equation}
\label{eq:cca-connection}
\theta_\alpha (\mbX_i, \mbX_j) = \min_{\mbQ \in \cO} ~ \arccos \frac{\langle \mbX_i^{\phi_\alpha}, \mbX_j^{\phi_\alpha} \mbQ \rangle}{\Vert \mbX_i^{\phi_\alpha} \Vert \Vert \mbX_j^{\phi_\alpha} \Vert} = \arccos ( \textstyle\sum_\ell \rho_\ell ) \, ,
\end{equation}
and we can see from \Cref{proposition:ang-shape-metric} that this defines a metric for any $0 \leq \alpha \leq 1$.
When $\alpha=0$, the values $\rho_1, \hdots, \rho_n$ are proportional to the canonical correlation coefficients, with $1/n$ being the factor of proportionality.
When $\alpha > 0$, these values can be viewed as ridge regularized canonical correlation coefficients \cite{Vinod1976}.
See Supplement C for further details.
Past works \cite{Raghu2017,Morcos2018} have used the average canonical correlation as a measure of representational similarity.
When $\alpha=0$, the average canonical correlation is given by $\sum_\ell \rho_\ell = \cos \theta_0(\mbX_i, \mbX_j)$.
Thus, if we apply $\arccos(\cdot)$ to the average canonical correlation, we modify the calculation to produce a proper metric (see \cref{fig4}A).
Since the covariance is often ill-conditioned or singular in practice, setting $\alpha > 0$ to regularize the calculation is also typically necessary.

\paragraph{Nonlinear invariances}
We discuss feature maps that enable nonlinear notions of equivalence, and which relate to kernel CCA \cite{Lai2000} and CKA \cite{Kornblith2019}, in Supplement C.



\subsection{Metrics for convolutional layers}
\label{subsec:convolutional-shape-metrics}

In deep networks for image processing, each convolutional layer produces a $h \times w \times c$ array of activations, whose axes respectively correspond to image height, image width, and channels (number of convolutional filters).
If stride-1 circular convolutions are used, then applying a circular shift along either spatial dimension produces the same shift in the layer's output.
It is natural to reflect this property, known as translation equivariance \cite{Cohen2016}, in the equivalence relation on layer representations.
Supposing that the feature map preserves the shape of the activation tensor, we have $\mbX_k^\phi \in \reals^{m \times h \times w \times c}$ for neural networks indexed by $k \in 1, \hdots, K$.
Letting $\cS(n)$ denote the group of $n$-dimensional circular shifts (a subgroup of the permutation group) and `$\otimes$' denote the Kronecker product, we propose:
\begin{equation}
\label{eq:circ-convolution-equivalence}
\mbX_i^\phi \sim \mbX_j^\phi
\iff
\mathrm{vec}(\mbX_i^\phi) = (\mbI \otimes \mbS_1 \otimes \mbS_2 \otimes \mbQ) \mathrm{vec}(\mbX_j^\phi)
\end{equation}
for some $\mbS_1 \in \cS(h)$, $\mbS_2 \in \cS(w)$, $\mbQ \in \cO(c)$, as the desired equivalence relation.
This relation allows for orthogonal invariance across the channel dimension but only shift invariance across the spatial dimensions.
The mixed product property of Kronecker products, $(\mbA \otimes \mbB)(\mbC \otimes \mbD) = \mbA \mbB \otimes \mbC \mbD$, ensures that the overall transformation maintains the group structure and remains an isometry.
\Cref{fig2}B uses a toy dataset (stacked MNIST digits) to show that this metric is sensitive to differences in spatial activation patterns, but insensitive to coherent spatial translations across channels.
In contrast, metrics that ignore the convolutional structure (as in past work \cite{Raghu2017,Kornblith2019}) treat very different spatial patterns as identical representations.

\begin{figure}
\centering
\includegraphics[width=\linewidth]{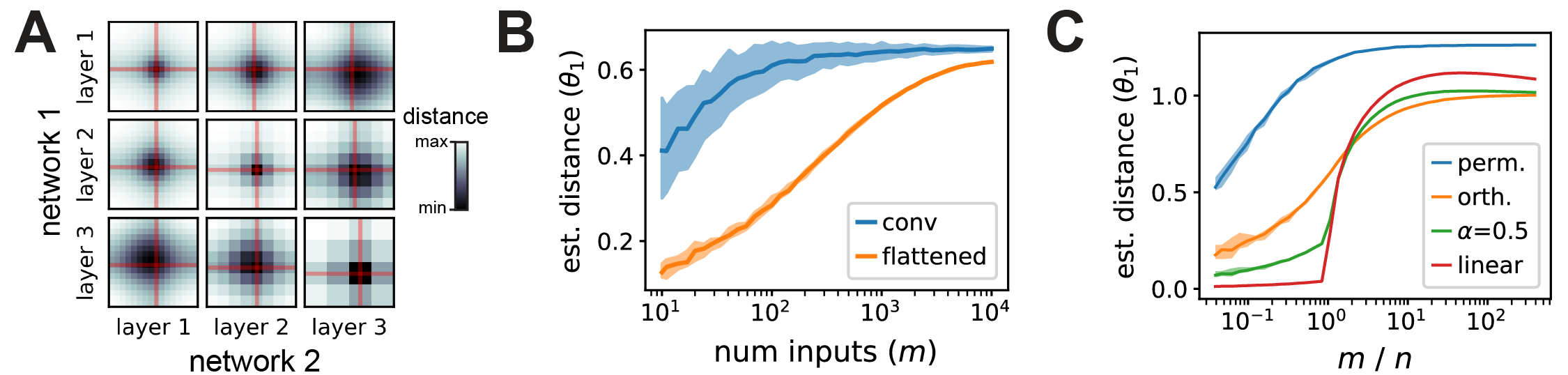}
\caption{
\textit{(A)} Each heatmap shows a brute-force search over the shift parameters along the width and height dimensions of a pair of convolutional layers compared across two networks.
The optimal shifts are typically close to zero (red lines).
\textit{(B)} Impact of sample size, $m$, on flattened and convolutional metrics with orthogonal invariance.
The convolutional metric approaches its final value faster than the flattened metric, which is still increasing even at the full size of the CIFAR-10 test set ($m=10^4$).
\textit{(C)} Impact of sample density, $m / n$, on metrics invariant to permutation, orthogonal, regularized linear ($\alpha = 0.5$), and linear transformations. Shaded regions mark the 10\textsuperscript{th} and 90\textsuperscript{th} percentiles across shuffled repeats. Further details are provided in Supplement E.
}
\label{fig3}
\end{figure}

Evaluating \cref{eq:circ-convolution-equivalence} requires optimizing over spatial shifts in conjuction with solving a Procrustes alignment.
If we fit the shifts by an exhaustive brute-force search, the overall runtime is $O(mh^2w^2c^2 + hwc^3)$, which is costly if this calculation is repeated across a large collection of networks.
In practice, we observe that the optimal shift parameters are typically close to zero (\cref{fig3}A).
This motivates the more stringent equivalence relation:
\begin{equation}
\label{eq:strict-convolution-equivalence}
\mbX_i^\phi \sim \mbX_j^\phi
\iff
\mathrm{vec}(\mbX_i^\phi) = (\mbI \otimes \mbI \otimes \mbI \otimes \mbQ) \mathrm{vec}(\mbX_j^\phi)
\hspace{1em}
\text{for some} 
\hspace{.4em}
\mbQ \in \cQ,
\end{equation}
which has a more manageable runtime of $O(mhwc^2 + c^3)$.
To evaluate the metrics implied by \cref{eq:strict-convolution-equivalence}, we can simply reshape each $\mbX_k^\phi$ from a $(m \times h \times w \times c)$ tensor into a $(mhw \times c)$ matrix and apply the Procrustes alignment procedure as done above for previous metrics.
In contrast, the ``flattened metric'' in \cref{fig2}B reshapes the features into a $(m \times hwc)$ matrix, resulting in a more computationally expensive alignment that runs in $O(mh^2w^2c^2 + h^3w^3c^3)$ time.

\subsection{How large of a sample size is needed?}

An important issue, particularly in neurobiological applications, is to determine the number of network inputs, $m$, and neurons, $n$, that one needs to accurately infer the distance between two network representations \cite{Shi2019}.
Reasoning about these questions rigorously requires a probabilistic perspective of neural representational similarity, which is missing from current literature and which we outline in Supplement D for generalized shape metrics.
Intuitively, looser equivalence relations are achieved by having more flexible alignment operations (e.g. nonlinear instead of linear alignments).
Thus, looser equivalence relations require more sampled inputs to prevent overfitting.
\Cref{fig3}B-C show that this intuition holds in practice for data from deep convolutional networks.
Metrics with looser equivalence relations---the ``flattened'' metric in panel B, or e.g. the linear metric in panel C---converge slower to a stable estimate as $m$ is increased.

\subsection{Modeling approaches and conceptual insights}
\label{subsec:insights}

Generalized shape metrics facilitate several new modeling approaches and conceptual perspectives.
For example, a collection of representations from $K$ neural networks can, in certain cases, be interpreted and visualized as $K$ points on a smooth manifold (see \cref{fig1}).
This holds rigorously due to the \textit{quotient manifold theorem} \cite{Lee2013} so long as $\cG$ is not a finite set (e.g. corresponding to permutation) and all matrices are full rank in the feature space.
This geometric intuition can be made even stronger when $\cG$ corresponds to a connected manifold, such as $\cS\cO(p)$.
In this case, it can be shown that the geodesic distance between two neural representations coincides with the metrics we defined in \Cref{proposition:euc-shape-metric,proposition:ang-shape-metric} (see Supplement C, and \cite{dryden2016}).
This result extends the well-documented manifold structure of \textit{Kendall's shape space} \cite{Kendall1984}.

Viewing neural representations as points on a manifold is not a purely theoretical exercise---several models can be adapted to manifold-valued data (e.g. principal geodesic analysis \cite{Fletcher2007} provides a generalization of PCA), and additional adaptions are an area of active research \cite{geomstats}.
However, there is generally no simple connection between these curved geometries and the flat geometries of Euclidean or Hilbert spaces \cite{Feragen2015}.\footnote{However, see \cite{Feragen2016} for a conjectured relationship and \cite{Jayasumana2013} for a result in the special case of 2D shapes.}
Unfortunately, the majority of off-the-shelf machine learning tools are incompatible with the former and require the latter.
Thus, we can resort to a heuristic approach: the set of $K$ representations can be embedded into a Euclidean space that approximately preserves the pairwise shape distances.
One possibility, employed widely in shape analysis, is to embed points in the tangent space of the manifold at a reference point \cite{Dryden1993,Rohlf1999}.
Another approach, which we demonstrate below with favorable results, is to optimize the vector embedding directly via multi-dimensional scaling \cite{Borg2005,Agrawal2021}.

\section{Applications and Results}
\label{sec:applications}

We analyzed two large-scale public datasets spanning neuroscience (Allen Brain Observatory, ABO; Neuropixels - visual coding experiment; \cite{Siegle2021}) and deep learning (NAS-Bench-101; \cite{Ying19}).
We constructed the ABO dataset by pooling recorded neurons from $K=48$ anatomically defined brain regions across all sessions; each $\mbX_k \in \reals^{m \times n}$ was a dimensionally reduced matrix holding the neural responses (summarized by $n=100$ principal components) to $m=1600$ movie frames (120 second clip, ``natural movie three'').
The full NAS-Bench-101 dataset contains $423{,}624$ architectures; however, we analyze a subset of $K=2000$ networks for simplicity.
In this application each $\mbX_k \in \reals^{m \times n}$ is a representation from a specific network layer, with $(m, n) \in \{(32^2 \times 10^5, 128), (16^2 \times 10^5, 256), (8^2 \times 10^5, 512), (10^5, 512)\}$. Here, $n$ corresponds to the number of channels and $m$ is the product of the number of test set images ($10^5$) and the height and width dimensions of the convolutional layer---i.e., we use equivalence relation in \cref{eq:strict-convolution-equivalence} to evaluate dissimilarity.

\begin{figure}
\centering
\includegraphics[width=\linewidth]{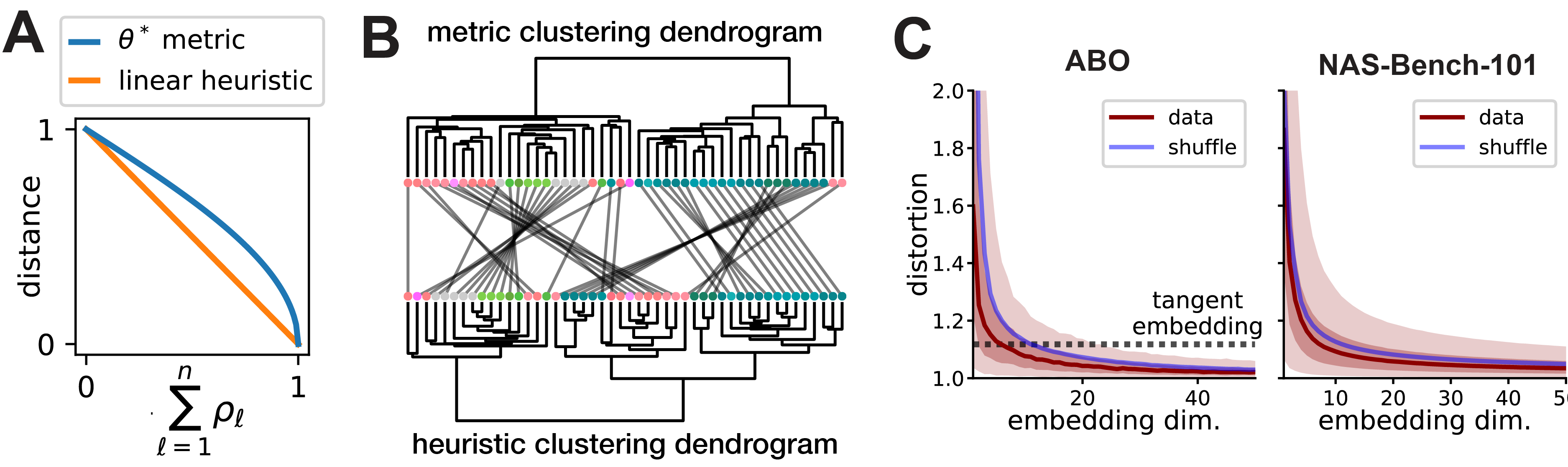}
\caption{
(A) Comparison of metric and linear heuristic.
(B) Metric and linear heuristic produce discordant hierarchical clusterings of brain areas in the ABO dataset.
Leaves represent brain areas that are clustered by representational similarity (see \cref{fig1}C), colored by Allen reference atlas, and ordered to maximize dendrogram similarities of adjacent leaves.
In the middle, grey lines connect leaves corresponding to the same brain region across the two dendrograms.
(C) ABO and NAS-Bench-101 datasets can be accurately embedded into Euclidean spaces. Dark red line shows median distortion.
Light red shaded region corresponds to 5th to 95th percentiles of distortion, dark red shaded corresponds to interquartile range.
The mean distortion of a null distribution over representations (blue line) was generated by shuffling the $m$ inputs independently in each network.
}
\label{fig4}
\end{figure}

\textbf{Triangle inequality violations can occur in practice when using existing methods.}
As mentioned above, a dissimilarity measure based on the mean canonical correlation, $1 - \sum_\ell \rho_\ell$, has been used in past work \cite{Morcos2018,Maheswaranathan2019}.
We refer to this as the ``linear heuristic.''
A slight reformulation of this calculation, $\arccos\, (\sum_\ell \rho_\ell)$, produces a metric that satisfies the triangle inequality (see \cref{eq:cca-connection}).
\Cref{fig4}A compares these calculations as a function of the average (regularized) canonical correlation: one can see that $\arccos(\cdot)$ is approximately linear when the mean correlation is near zero, but highly nonlinear when the mean correlation is near one.
Thus, we reasoned that triangle inequality violations are more likely to occur when $K$ is large and when many network representations are close to each other.
Both ABO and NAS-Bench-101 datasets satisfy these conditions, and in both cases we observed triangle inequality violations by the linear heuristic with full regularization ($\alpha=1$): 17/1128 network pairs in the ABO dataset had at least one triangle inequality violation, while 10128/100000 randomly sampled network pairs contained violations in the NAS-Bench-101 Stem layer dataset.
We also examined a standard version of RSA that quantifies similarity via Spearman's rank correlation coefficient \cite{Kriegeskorte2008}.
Similar to the results above, we observed violations in 14/1128 pairs of networks in the ABO dataset.

Overall, these results suggest that generalized shape metrics correct for triangle inequality violations that do occur in practice.
Depending on the dataset, these violations may be rare (\textapprox1\% occurrence in ABO) or relatively common (\textapprox10\% in the Stem layer of NAS-Bench-101).
These differences can produce quantitative discrepancies in downstream analyses.
For example, the dendrograms produced by hierarchical clustering differ depending on whether one uses the linear heuristic or the shape distance (\textapprox85.1\% dendrogram similarity as quantified by the method in \cite{Gates2019}; see \cref{fig4}B).

\textbf{Neural representation metric spaces can be approximated by Euclidean spaces.}
Having established that neural representations can be viewed as elements in a metric space, it is natural to ask if this metric space is, loosely speaking, ``close to'' a Euclidean space.
We used standard multidimensional scaling methods (SMACOF, \cite{Borg2005}; implementation in \cite{sklearn}) to obtain a set of embedded vectors, $\mby_i \in \reals^L$, for which $\theta_1(\mbX_i^\phi, \mbX_j^\phi) \approx \Vert \mby_i - \mby_j \Vert$ for $i, j \in 1, \hdots, K$.
The embedding dimension $L$ is a user-defined hyperparameter.
This problem admits multiple formulations and optimization strategies \cite{Agrawal2021}, which could be systematically explored in future work.
Our simple approach already yields promising results: we find that moderate embedding dimensions ($L \approx 20$) is sufficient to produce high-quality embeddings.
We quantify the embedding distortions multiplicatively \cite{Vankadara2018}:
\begin{equation}
\max \big ( \theta_1(\mbX_i^\phi, \mbX_j^\phi) / \Vert \mby_i - \mby_j \Vert; ~ \Vert \mby_i - \mby_j \Vert/ \theta_1(\mbX_i^\phi, \mbX_j^\phi) \big)
\end{equation}
for each pair of networks $i, j \in 1, \hdots K$.
Plotting the distortions as a function of $L$ (\cref{fig4}C), we see that they rapidly decrease, such that 95\% of pairwise distances are distorted by, at most, \textapprox5\% (ABO data) or 10\% (NAS-Bench-101) for sufficiently large $L$.
Past work \cite{Maheswaranathan2019} has used multidimensional scaling heuristically to visualize collections of network representations in $L=2$ dimensions.
Our results here suggest that such a small value of $L$, while being amenable to visualization, results in a highly distorted embedding.
It is noteworthy that the situation improves dramatically when $L$ is even modestly increased.
While we cannot easily visualize these higher-dimensional vector embeddings, we can use them as features for downstream modeling tasks.
This is well-motivated as an approximation to performing model inference in the true metric space that characterizes neural representations \cite{Vankadara2018}.

\begin{figure}
\centering
\includegraphics[width=\linewidth]{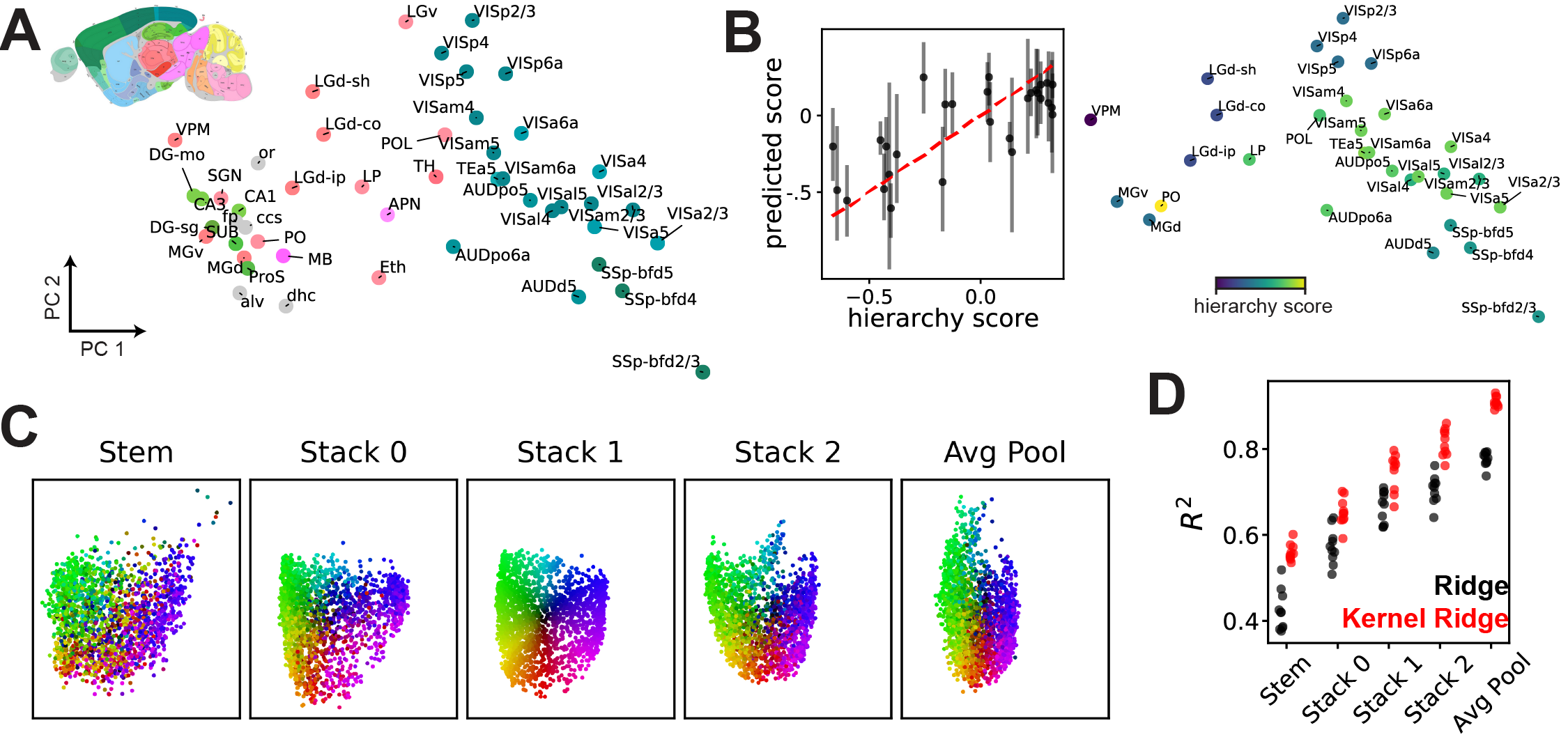}
\caption{
(A) PCA visualization of representations across 48 brain regions in the ABO dataset.
Areas are colored by the reference atlas (see inset), illustrating a functional clustering of regions that maps onto anatomy.
(B) \textit{Left}, kernel regression predicts anatomical hierarchy \cite{Harris2019} from embedded representations (see Supplement E).
\textit{Right}, PCA visualization of 31 areas labeled with hierarchy scores.
(C) PCA visualization of 2000 network representations (a subset of NAS-Bench-101) across five layers, showing global structure is preserved across layers.
Each network is colored by its position in the ``Stack 1'' layer (the middle of the architecture).
(D) Embeddings of NAS-Bench-101 representations are predictive of test set accuracy, \textit{even in very early layers}.
}
\label{fig5}
\end{figure}

\paragraph{Anatomical structure and hierarchy is reflected in ABO representations.}
We can now collect the $L$-dimensional vector embeddings of $K$ network representations into a matrix $\mbZ \in \reals^{K \times L}$.
The results in \cref{fig4}C imply that the distance between any two rows, $\Vert \mbz_i - \mbz_j\Vert$, closely reflects the distance between network representations $i$ and $j$ in shape space.
We applied PCA to $\mbZ$ to visualize the $K=48$ brain regions and found that anatomically related brain regions indeed were closer together in the embedded space (\cref{fig5}A): cortical and sub-cortical regions are separated along PC 1, and different layers of the same region (e.g. layers 2/3, 4, 5, and 6a of VISp) are clustered together.
As expected from \cref{fig4}C, performing multidimensional scaling directly to a low-dimensional space ($L=2$, as done in \cite{Maheswaranathan2019}) results in a qualitatively different outcome with distorted geometry (see Supplement E).
Additionally, we used $\mbZ$ to fit an ensembled kernel regressor to predict an anatomical hierarchy score (defined in \cite{Harris2019}) from the embedded vectors (\cref{fig5}B).
Overall, these results demonstrate that the geometry of the learned embedding is scientifically interpretable and can be exploited for novel analyses, such as nonlinear regression.
To our knowledge, the fine scale anatomical parcellation used here is novel in the context of representational similarity studies.

\textbf{NAS-Bench-101 representations show persistent structure across layers.} 
Since we collected representations across five layers in each deep network, the embedded representation vectors form a set of five $K \times L$ matrices, $\{\mbZ_1, \mbZ_2, \mbZ_3, \mbZ_4, \mbZ_5\}$.
We aligned these embeddings by rotations in $\reals^L$ via Procrustes analysis, and then performed PCA to visualize the $K=2000$ network representations from each layer in a common low-dimensional space.
We observe that many features of the global structure are remarkably well-preserved---two networks that are close together in the \texttt{Stack1} layer are assigned similar colors in \cref{fig5}C, and are likely to be close together in the other four layers.
This preservation of representational similarity across layers suggests that even early layers contain signatures of network performance, which we expect to be present in the \texttt{AvgPool} layer.
Indeed, when we fit ridge and RBF kernel ridge regressors to predict test set accuracy from representation embeddings, we see that even early layers support moderately good predictions (\cref{fig5}D).
This is particularly surprising for the \texttt{Stem} layer.
This is the first layer in each network, and its architecture is identical for all networks.
Thus, the differences that are detected in the \texttt{Stem} layer result only from differences in backpropagated gradients.
Again, these results demonstrate the ability of generalized shape metrics to incorporate neural representations into analyses with greater scale ($K$ corresponding to thousands of networks) and complexity (nonlinear kernel regression) than has been previously explored.

\section{Conclusion and Limitations}

We demonstrated how to ground analyses of neural representations in proper metric spaces.
By doing so, we capture a number of theoretical advantages \cite{Yianilos1993,Dasgupta2005,Baraty2011,Wang2015,Chang2016}.
Further, we suggest new practical modeling approaches, such as using Euclidean embeddings to approximate the representational metric spaces.
An important limitation of our work, as well as the past works we build upon, is the possibility that representational geometry is only loosely tied to higher-level algorithmic principles of network function \cite{Maheswaranathan2019}.
On the other hand, analyses of representational geometry may provide insight into lower-level implementational principles \cite{Hamrick2020}.
Further, these analyses are highly scalable, as we demonstrated by analyzing thousands of networks---a much larger scale than is typically considered.

We used simple metrics (extensions of regularized CCA) in these analyses, but metrics that account for nonlinear transformations across neural representations are also possible as we document in Supplement C.
The utility of these nonlinear extensions remains under-investigated and it is possible that currently popular linear methods are insufficient to capture structures of interest.
For example, the topology of neural representations has received substantial interest in recent years \cite{Rybakken2019,Chaudhuri2019,Rouse2021,Gardner2021}.
Generalized shape metrics do not directly capture these topological features, and future work could consider developing new metrics that do so.
A variety of recent developments in topological data analysis may be useful towards this end \cite{Kusano2018,Moor2020,Jensen2020}.

Finally, several of the metrics we described can be viewed as geodesic distances on Riemannian manifolds \cite{Kendall1984}.
Future work would ideally exploit methods that are rigorously adapted to such manifolds, which are being actively developed \cite{geomstats}.
Nonetheless, we found that optimized Euclidean embeddings, while only approximate, provide a practical off-the-shelf solution for large-scale surveys of neural representations.

\subsection*{Acknowledgments}

We thank Ian Dryden (Florida International University), S{\o}ren Hauberg (Technical University of Denmark), and Nina Miolane (UC Santa Barbara) for fruitful discussions.
A.H.W. was supported by the National Institutes of Health BRAIN initiative (1F32MH122998-01), and the Wu Tsai Stanford Neurosciences Institute Interdisciplinary Scholar Program.
E. K. was supported by the Wu Tsai Stanford Neurosciences Institute Interdisciplinary Graduate Fellows Program.
S.W.L. was supported by grants from the Simons Collaboration on the Global Brain (SCGB 697092) and the NIH BRAIN Initiative (U19NS113201 and R01NS113119).

\printbibliography
\clearpage

\newgeometry{margin=0.75in}

\begin{center}\Large\bfseries
Supplemental Information:\\
Generalized Shape Metrics on Neural Representations
\end{center}

This supplement is organized into five sections.
First, in Supplement~\ref{appendix:a_background}, we review background material on metric spaces and other relevant mathematical concepts.
In Supplement~\ref{appendix:b_main_propositions}, we prove the two propositions that appear in the main text.
Supplement~\ref{appendix:c_additional_theory} collects together several miscellaneous results which demonstrate that generalized shape metrics include similarity measures based on CCA, kernel CCA, and geodesic distance on Kendall's shape space.
In Supplement~\ref{appendix:d_stochastic_layers}, we outline an extension and reinterpretation of generalized shape metrics to stochastic random variables.
This extension represents a rich opportunity for future research and also provides a better foundation to interpret the results presented in Fig. 3 of the main text, which empirically characterize the number of images needed to estimate the distance between two neural networks.
Finally, in Supplement~\ref{appendix:e_experimental_methods}, we collect additional methodological details about the experiments we present in the main text.
\counterwithin{figure}{subsection}

\begin{appendices}

\section{Background}
\label{appendix:a_background}

\subsection{Notation}
\label{appendix:a_background_notation}

Vectors in real coordinate space are denoted in boldface with lowercase letters, e.g. $\mbx \in \reals^n$.
Matrices are denoted in boldface with uppercase letters, e.g. $\mbX \in \reals^{m \times n}$.
We use the same notation to denote linear operators, e.g. $\mbT \in \cG$ where $\cG$ is a set of linear operators.

Letters in regular type face, e.g. $x$ or $X$, may denote scalars or elements of some abstract vector space, with the distinction being made clear from context.
For example, the space of random variables with outcomes over $\reals^n$ defines a vector space that we will see is compatible with the basic framework of generalized shape metrics.
This extension of shape metrics to stochastic layers and neural responses is outlined in Supplement~\ref{appendix:d_stochastic_layers}.

If $\mbT$ is a linear operator on some vector space, and $X$ is a vector within this space, we will use $\mbT X$ to denote the transformation $X \mapsto \mbT(X)$.
Further, if $\mbT_1$ and $\mbT_2$ are linear operators, we write $\mbT_1 \mbT_2 X$ in place of $\mbT_1(\mbT_2(X))$, and we use $\mbT_1 \mbT_2$ to denote the composition of the two linear operators.
These notational choices intuitively draw parallels with matrix-vector and matrix-matrix multiplication, respectively.

\subsection{Metrics}

Here we revisit our definition of a metric given in the main text to provide more rigorous details and clarify the role of the equivalence relation.

\begin{definition}
A \textbf{metric} on a set $\cS$ is a function $\cS \times \cS \mapsto \reals_+$, which satisfies, for all $X, Y, M \in \cS$, the following three conditions:
\begin{itemize}
\item \textit{Identity.} $d(X, Y) = 0$ if and only if $X = Y$
\item \textit{Symmetry.} $d(X, Y) = d(Y, X)$
\item \textit{Triangle Inequality.} $d(X, Y) \leq d(X, M) + d(M, Y)$
\end{itemize}
\label{def:strict-metric}
\end{definition}

We have seen that it is useful to relax the first condition (\textit{Identity}) to an equivalence relation.
That is, rather than strict equality, we demand that $d(X, Y) = 0$ if and only if $X \sim Y$, for some specified equivalence relation $\sim$.
In this scenario, the distance function is \textit{not}, strictly speaking, a metric on $\cS$.
However, it still does define a metric on the appropriate \textit{quotient set}, which we now define.

\begin{definition}
Let $\sim$ denote an equivalence relation defined on some set $\cS$.
Then given any $M \in \cS$, we can define the set of all elements equivalent to $M$ as $\{X \in \cS \mid X \sim M\}$, which is called the \textbf{equivalence class} of the element $M$.
The set of all equivalence classes, denoted $\cS / \sim$, is called the \textbf{quotient set} of $\cS$ with respect to the specified equivalence relation.
\end{definition}

For example, the Euclidean distance $\Vert \mbx - \mby \Vert$ is a metric on the set of vectors in $\reals^n$.
The angular distance $\arccos(\mbx^\top \mby/ \sqrt{\mbx^\top \mbx \cdot \mby^\top \mby})$ is not a metric on $\reals^n$, but it defines a metric between sets of points contained in rays emanating from the origin (i.e. points in $\reals^n$ with an equivalence relation given by nonnegative scaling).
These technical distinctions above are not central to our story, so we will often refer to a function as a ``metric'' without explicitly defining what set it acts upon. 
In all cases, it should be understood as the quotient set defined by the specified equivalence relation.

\subsection{Hilbert spaces}

A \textbf{vector space} $\cH$ is a collection of objects (called vectors) that are equipped with two operations: vector addition (given $X \in \cH$ and $Y \in \cH$ we have $X + Y \in \cH$) and scalar multiplication (given $X \in \cH$ and $\alpha \in \reals$ we have $\alpha X \in \cH$).
An \textbf{inner product space} is a vector space that is additionally equipped with a function $\cH \times \cH \mapsto \reals$, called the \textit{inner product}, which is denoted with angle brackets $\langle \cdot, \cdot \rangle$ and satisfies:
\begin{itemize}
\item \textit{Symmetry.} $\langle X, Y \rangle = \langle Y, X \rangle$
\item \textit{Linearity.} $\langle Z + \alpha X, Y \rangle = \langle Z, Y \rangle + \alpha \langle X, Y \rangle$
\item \textit{Positive Definiteness.} $\langle X, X \rangle \geq 0 ~~\text{with equality if and only if}~~ X = 0$
\end{itemize}
A \textbf{Hilbert space} is an inner product space that satisfies an additional technical requirement (all Cauchy sequences of vectors in $\cH$ converge to a limit in $\cH$).

The set of vectors in $\reals^n$ defines a Hilbert space, where the inner product corresponds to the usual dot product.
Similarly, the set of matrices in $\reals^{m \times n}$, equipped with the Frobenius inner product $\langle \mbX, \mbY \rangle = \Tr [\mbX^\top \mbY]$ also defines a Hilbert space.
In Supplement~\ref{appendix:d_stochastic_layers},  we will exploit the fact that random vectors over $\reals^n$ also define a Hilbert space where the inner product is given by the expectation of the dot product.
This enables us to extend the framework of generalized shape metrics to stochastic neural layers.

\subsection{Euclidean and Angular Distances in Hilbert Spaces}

One of the most fundamental properties of a Hilbert space is the \textbf{Cauchy-Schwarz inequality},
\begin{equation}
| \langle X, Y \rangle | \leq \Vert X \Vert \Vert Y \Vert \quad \text{for all} \quad (X, Y) \in \cH \times \cH,
\end{equation}
which can be derived from the properties of the inner product.
Using this, we can verify that the norm is sub-additive:
\begin{equation}
\Vert X + Y \Vert  \leq \Vert X \Vert + \Vert Y \Vert  \quad \text{for all} \quad (X, Y) \in \cH \times \cH,
\end{equation}
Defining $d_\text{euc}(X, Y) = \Vert X - Y \Vert$ to be the generalization of Euclidean distance to Hilbert spaces, we see that triangle inequality follows immediately:
\begin{equation}
\label{eq:appendix_a_euc_distance}
d_\text{euc}(X, Y) = \Vert X - Y \Vert = \Vert X - M + M - Y \Vert \leq \Vert X - M \Vert + \Vert M - Y \Vert = 
d_\text{euc}(X, M) + d_\text{euc}(M, Y)
\end{equation}
for all choices of $X$, $Y$, and $M$ in $\cH$.
Euclidean distance evidently satisfies the remaining two properties of a metric---symmetry and nonnegativity. 

The angular distance is defined as:
\begin{equation}
\label{eq:appendix_a_ang_distance}
d_\theta(X, Y) = \arccos \left [ \frac{\langle X, Y \rangle}{\Vert X \Vert \Vert Y \Vert} \right ]
\end{equation}
The Cauchy-Schwarz inequality implies that the argument to $\arccos(\cdot)$ is always within its domain (i.e. on the interval $[-1, 1]$).
The angular distance is a metric over equivalence classes defined by nonnegative scaling: formally, $X \sim Y$ if and only if there exists an $s > 0$ such that $X = s Y$.
Geometrically, one can think of $d_\theta(X, Y)$ as the geodesic path length between points on a sphere.
Intuitively, this is nonnegative, symmetric, and obeys the triangle inequality.
We provide a short proof that the triangle inequality is indeed satisfied below.

\begin{proof}[\protect{\textbf{Proof:} Angular distance satisfies the triangle inequality}]
Consider three unit-norm vectors: $X$, $Y$, and $Z$.
The triangle inequality trivially holds if any pair of $X$, $Y$, and $Z$ are equal, so we can assume $X$, $Y$, and $Z$ are distinct.
Now define two vectors $U$ and $V$ as follows:
\begin{align}
U &= X - Y \langle X, Y \rangle
\\
V &= Z - Y \langle Z, Y \rangle
\end{align}
Note that $\langle U, Y \rangle = 0$ and $\langle V, Y \rangle = 0$.
Thus, we can interpret $U$ as the part of $X$ that is orthogonal to $Y$.
Likewise, we can interpret $V$ as the part $Z$ that is orthogonal to $Y$.
Further, we have:
\begin{align}
X &= Y \langle X, Y \rangle + U \langle X, U \rangle = Y \cos \theta_{XY} + U \sin \theta_{XY}
\\
Z &= Y \langle Z, Y \rangle + V \langle Z, V \rangle = Y \cos \theta_{ZY} + V \sin \theta_{ZY}
\end{align}
where we introduced the shorthand $\theta_{XY} = d_\theta(X, Y)$ for concision.
Now,
\begin{align}
\cos \theta_{XZ} = \langle X, Z \rangle
&= \langle Y \cos \theta_{XY} + U \sin \theta_{XY}, Y \cos \theta_{ZY} + V \sin \theta_{ZY} \rangle
\\
&= \cos \theta_{XY} \cos \theta_{ZY} + \langle U, V \rangle \sin \theta_{XY} \sin \theta_{ZY}
\label{eq:angular-metric-proof-semikey-step}
\\
&\geq \cos \theta_{XY} \cos \theta_{ZY} - \sin \theta_{XY} \sin \theta_{ZY}
\label{eq:angular-metric-proof-key-step}
\\
&=\cos(\theta_{XY} + \theta_{ZY})
\label{eq:angular-metric-proof-trig-identity}
\end{align}
On line (\ref{eq:angular-metric-proof-semikey-step}), many terms simplify since $\langle Y, Y \rangle = 1$, and $\langle U, Y \rangle = \langle V, Y \rangle = 0$.
To introduce the inequality on line (\ref{eq:angular-metric-proof-key-step}), notice that the Cauchy-Schwarz inequality implies $\langle U, V \rangle \geq -1$.
Thus, replacing $\langle U, V \rangle$ with $-1$ produces a lower bound on $\cos \theta_{XZ}$ since $\sin \theta_{XY} \sin \theta_{ZY} \geq 0$.
The final step on line (\ref{eq:angular-metric-proof-trig-identity}) applies an elementary identity from trigonometry.
Overall, we have $\cos \theta_{XZ} \geq \cos(\theta_{XY} + \theta_{ZY})$.
This directly implies the desired triangle inequality, $\theta_{XZ} \leq \theta_{XY} + \theta_{ZY}$, since $\arccos(\cdot)$ is a monotonically decreasing function.
\end{proof}

\subsection{The Orthogonal Group}
\label{subsec:the-orthogonal-group}

Another important feature of Hilbert spaces is the notion of an orthogonal transformation.
These are linear transformations which preserve the inner product.
Below, we also define the familiar transpose operator for a general Hilbert space.

\begin{definition}
An \textbf{orthogonal transformation} on a Hilbert space $\cH$ is any linear transformation $\mbQ$, which satisfies $\langle \mbQ X, \mbQ Y \rangle = \langle X, Y \rangle$ for any choice of $X \in \cH$ and $Y \in \cH$.
\end{definition}

\begin{definition}
Let $\mbW : \cV \mapsto \cV$ be a linear transformation on a Hilbert space $\cV$.
For any choice of $\mbW$, there is a unique linear transformation $\mbW^\top$, called the \textbf{transpose} (or \textit{adjoint}) of $\mbW$, which is denoted $\mbW^\top$ and which satisfies $\langle \mbW X, Y \rangle = \langle X, \mbW^\top Y \rangle$ for any choice of $X \in \cV$ and $Y \in \cV$.
\end{definition}

Let $\mbQ$ be orthogonal.
Since $\langle X, Y \rangle = \langle \mbQ X, \mbQ Y \rangle = \langle X, \mbQ^\top \mbQ Y \rangle$, we see that $\mbQ^\top \mbQ$ is the identity transformation and thus $\mbQ^\top$ and $\mbQ$ are inverses.
One can show that these inverses commute, and thus $\mbQ^\top$ is also orthogonal since $\langle \mbQ^\top X, \mbQ^\top Y \rangle = \langle \mbQ \mbQ^\top X, \mbQ \mbQ^\top Y \rangle = \langle X, Y \rangle$.
Finally, let $\mbQ_1$ and $\mbQ_2$ be any pair of orthogonal transformations on $\cV$.
Then, the composition of these transformations $\mbQ_2 \mbQ_1$ is evidently orthogonal, since:
$\langle X, Y \rangle = \langle \mbQ_1 X, \mbQ_1 Y \rangle = \langle \mbQ_2 \mbQ_1 X, \mbQ_2 \mbQ_1 Y \rangle$.

In summary, we have just shown that the inverse of every orthogonal matrix is also orthogonal and orthogonal transformations are closed under composition.
This shows that the set of orthogonal transformations on a Hilbert space fulfills the axioms of a \textbf{group}, as defined below:

\begin{definition}
A \textbf{group} is a set $\cG$ equipped with a binary operation that maps two elements of $\cG$ onto another element of $\cG$, which satisfies:
\begin{enumerate}
    \item Associativity: For all $\mbT_1, \mbT_2, \mbT_3$ in $\cG$, one has $(\mbT_1 \mbT_2) \mbT_3 = \mbT_1 (\mbT_2 \mbT_3)$.
    \item Identity element: There exists a unique element $\mbI \in \cG$ such that $\mbI \mbT = \mbT \mbI = \mbT$ for all $\mbT \in \cG$.
    \item Invertibility: For every $\mbT \in \cG$ there exists another element $\mbT^{-1} \in \cG$ such that $\mbT \mbT^{-1} = \mbT^{-1} \mbT = \mbI$.
\end{enumerate}
\end{definition}

Here, we are only interested in groups of \textit{linear functions}, so the ``binary operation'' referred to above is function composition (see \cref{appendix:a_background_notation} for notational conventions regarding linear operators).

We are particularly interested in groups of (linear) transformations that preserve distances.
Such transformations are called (linear) \textbf{isometries}, which we define below.

\begin{definition}
Let $\cS$ be a set and let $d : \cS \times \cS \mapsto \reals_+$ be a metric on this set.
Then a transformation $\mbT : \cS \mapsto \cS$, is called an \textbf{isometry} on the metric space $(d, \cS)$ if $d(\mbT X, \mbT Y) = d(X, Y)$ for all $X, Y \in \cS$.
\end{definition}

It is easy to see that the orthogonal group is a group of isometries with respect to the (generalized) Euclidean and angular distance metrics.
For the Euclidean distance we have:
\begin{equation}
\begin{aligned}
d_\text{euc}^2(X, Y) = \Vert X - Y \Vert^2 &= \langle X, X \rangle + \langle Y, Y \rangle - 2 \langle X, Y \rangle \\
&= \langle \mbQ X, \mbQ X \rangle + \langle \mbQ Y, \mbQ Y \rangle - 2 \langle \mbQ X, \mbQ Y \rangle \\
&= \Vert \mbQ X - \mbQ Y \Vert^2 = d_\text{euc}^2(\mbQ X, \mbQ Y)
\end{aligned}
\end{equation}
For the angular distance we have:
\begin{equation}
\cos \left [ d_\theta(X, Y) \right ] = 
\frac{\langle X, Y \rangle}{ \Vert X \Vert \Vert Y \Vert} =
\frac{\langle \mbQ X, \mbQ Y \rangle}{ \Vert \mbQ X \Vert \Vert \mbQ Y \Vert} =
\cos \left [ d_\theta(\mbQ X, \mbQ Y) \right ]
\end{equation}

\section{Proof of Propositions 1 \& 2}
\label{appendix:b_main_propositions}

Both propositions in the main text follow immediately as special cases of the following result, which states that minimizing any metric over a group of isometries results in a metric on the corresponding quotient space.
After proving this result we conclude this section by briefly outlining these special cases.

\begin{proposition*}[A generalization of Propositions 1 \& 2]
Let $(g, \cH)$ be a metric space, where ${g : \cH \times \cH \mapsto \reals_+}$ denotes the distance function.
Let $\cG$ be a group of isometries on this metric space.
Then the function:
\begin{equation}
h(X, Y) = \min_{\mbT \in \cG} ~ g(X, \mbT Y)
\end{equation}
\end{proposition*}
defines a metric over the quotient space $\cH / \sim$ where the equivalence relation is $X \sim Y$ if and only if $X = \mbT Y$ for some $\mbT \in \cG$.
\begin{proof}
First, define $\mbT_{XY} = \argmin_{\mbT \in \cG} g(X, \mbT Y)$. So, $h(X, Y) = g(X, \mbT_{XY} Y)$.
Since $g$ is a metric, $g(X, Y) = 0$ if and only if $X = Y$.
Thus, $h(X, Y) = 0$ if and only if $X = \mbT_{XY} Y$, or equivalently if $X \sim Y$ by the stated equivalence relation.

Next, we prove that $h(X, Y) = h(Y, X)$.
By the group axioms, every element in $\cG$ is invertible by another element in the set, so $\mbT^{-1}_{XY} \in \cG$.
Further, every element of $\cG$ is an isometry with respect to $g$.
Thus,
\begin{equation}
h(X, Y) = g(X, \mbT_{XY} Y) = g(\mbT_{XY}^{-1} X,   \mbT_{XY}^{-1} \mbT_{XY} Y) = g(Y, \mbT_{XY}^{-1} X) \geq g(Y, \mbT_{YX} X) = h(Y, X) \, ,
\end{equation}
where the inequality follows from replacing $\mbT_{XY}^{-1}$ with the optimal $\mbT_{YX} = \argmin_{\mbT \in \cG} g(Y, \mbT X)$.
However, by the same chain of logic, we also have:
\begin{equation}
h(Y, X) = g(Y, \mbT_{YX} X) = g(\mbT_{YX}^{-1} Y,   \mbT_{YX}^{-1} \mbT_{YX} X) = g(X, \mbT_{YX}^{-1} Y) \geq g(X, \mbT_{XY} Y) = h(X, Y) \, .
\end{equation}
Thus, we have $h(X, Y) \geq h(Y, X)$, but also $h(Y, X) \geq h(X, Y)$.
We conclude $h(X, Y) = h(Y, X)$ and $\mbT_{XY}^{-1} = \mbT_{YX}$.

It remains to prove the triangle inequality.
This is done by the following sequence:
\begin{align}
h(X, Y) &= g(X, \mbT_{XY} Y) \\
&\leq g(X, \mbT_{XZ} \mbT_{ZY} Y) \\
&\leq g(X, \mbT_{XZ} Z) + g(\mbT_{XZ} Z, \mbT_{XZ} \mbT_{ZY} Y) \\
&= g(X, \mbT_{XZ} Z) + g(Z, \mbT_{ZY} Y) \\
&= h(X, Z) + h(Z, Y)
\end{align}
The first inequality follows from replacing the optimal alignment, $\mbT_{XY}$, with a sub-optimal alignment $\mbT_{XZ} \mbT_{ZY}$.
The second inequality follows from the triangle inequality on $g$, after choosing $\mbT_{XZ} Z$ as the midpoint.
The penultimate step follows from $\mbT_{XZ}$ being an isometry on $g$.
\end{proof}

\paragraph{Relation to Proposition 1}
The space $\cH$ corresponds to $\reals^{m \times p}$, which is equipped with the typical Frobenius inner product.
The distance function $g$ is Euclidean distance, see \cref{eq:appendix_a_euc_distance}.
The group $\cG$ corresponds to any group of linear isometries which can be expressed as a matrix multiplication on the right.
That is, any transformation from $\reals^{m \times p} \mapsto \reals^{m \times p}$ that can be expressed as $\mbX \mapsto \mbX \mbM$ for some $\mbM \in \reals^{p \times p}$.

\paragraph{Relation to Proposition 2}
The space $\cH$ corresponds to $\bbS^{m \times p}$ (the ``sphere'' of $m \times p$ matrices with unit Frobenius norm).
The distance function $g$ is the angular distance, see \cref{eq:appendix_a_ang_distance}.
The group $\cG$ is defined as done directly above in our discussion of Proposition 1.

\section{Connections to Other Methods}
\label{appendix:c_additional_theory}

This section describes the connections between generalized shape metrics and existing representational similarity measures in greater detail.
For simplicity, we consider quantifying the similarity between two networks with $n$ neurons or hidden layer units.
We use $\mbX \in \reals^{m \times n}$ and $\mbY \in \reals^{m \times n}$ to denote matrices holding the hidden layer activations of two networks over $m$ common test inputs.
In many cases, networks have distinct numbers of neurons or hidden units; however, this can be accommodated by applying PCA or zero-padding representations to achieve a common dimension.

For further simplicity, we will assume that $\mbX$ and $\mbY$ are mean-centered such that $\mbX^\top \boldsymbol{1}_n = \mbY^\top \boldsymbol{1}_n = \boldsymbol{0}_n$, where $\boldsymbol{0}_n$ and $\boldsymbol{1}_n$ respectively denote an $n$-dimensional vector of zeros and ones.
Intuitively, this mean-centering removes the effect of translations in neural activation space when computing distances between neural representations.
In the main text, we show this mean-centering step explicitly as a centering matrix $\mbC \in \reals^{m \times m}$ that is included in the feature map, $\phi$.
The mean-centering step is not strictly required, but is a typical preprocessing step in canonical correlations analysis \cite{Uurtio2017} and Procrustes analysis \cite{dryden2016}.

\subsection{Permutation Invariance \& Linear Assignment Problems}
\label{appendix:c_additional_theory_permutation_subsec}

Consider the problem of finding the best permutation matrix which matches two sets of neural activations in terms of Euclidean distance.
That is, we seek to find
\begin{equation}
\mbPi^* = \argmin_{\mbPi \in \cP} ~ \big\Vert \mbX - \mbY \mbPi \big\Vert
\label{eq:permutation-minimization} \, ,
\end{equation}
where $\cP$ is the set of $n \times n$ permutation matrices.
Note that this is equivalent to finding the permutation matrix that minimizes squared Euclidean distance, and that:
\begin{equation}
\big\Vert \mbX - \mbY \mbPi \big\Vert^2 = \langle \mbX, \mbX \rangle + \langle \mbY, \mbY \rangle  - 2\langle \mbX, \mbY \mbPi \rangle \, .
\label{eq:permutation-reformulation}
\end{equation}
Since $\langle \mbX, \mbX \rangle$ and $\langle \mbY, \mbY \rangle$ are constant terms, the minimization in (\ref{eq:permutation-minimization}) is equivalent to:
\begin{equation}
\mbPi^* = \argmin_{\mbPi \in \cP} ~ - 2\langle \mbX, \mbY \mbPi \rangle = \argmax_{\mbPi \in \cP} ~ \langle \mbX, \mbY \mbPi \rangle = \argmin_{\mbPi \in \cP} ~ d_\theta(\mbX, \mbY \mbPi) .
\end{equation}
The final equality holds since the angular distance is given by a monotonically decreasing function (i.e., $\arccos$) of the maximized inner product.
Finally, using the definition of the Frobenius inner product, $\langle \mbX, \mbY \mbPi \rangle = \Tr[\mbX^\top \mbY \mbPi]$, and so,
\begin{equation}
\mbPi^* = \argmax_{\mbPi \in \cP} ~ \Tr[\mbX^\top \mbY \mbPi] \, .
\label{eq:linear-assignment-problem}
\end{equation}
This final reformulation is the well-known \textit{linear assignment problem} \cite{Burkard2012}.
This can be solved efficiently in $O(n^3)$ time using standard algorithms \cite{Crouse2016}, which are readily available in standard scientific computing environments.
For example, the function \texttt{scipy.optimize.linear\_sum\_assignment} provides an implementation in Python \cite{scipy}.

\subsection{Orthogonal Procrustes Problems}
\label{subsec:orthogonal-procrustes-additional-notes}

Instead of optimizing over permutations, we may wish to optimize over orthogonal transformations.
Given two matrices $\mbX \in \reals^{m \times n}$ and $\mbY \in \reals^{m \times n}$, we seek to find
\begin{equation}
\mbQ^* = \argmin_{\mbQ \in \cO} ~ \Vert \mbX - \mbY \mbQ \Vert \, ,
\label{eq:classic-procrustes}
\end{equation}
where $\cO$ is the set of $n \times n$ orthogonal matrices.
This is known as the orthogonal Procrustes problem \parencite{Gower2004}.
Following the same steps as above in \cref{appendix:c_additional_theory_permutation_subsec}, we can see that $\mbQ^*$ also minimizes the angular distance between two matrices, and maximizes their inner product:
\begin{equation}
\mbQ^* = \argmax_{\mbQ \in \cO} ~ \langle \mbX, \mbY \mbQ \rangle = \argmin_{\mbQ \in \cO} ~ d_\theta( \mbX, \mbY \mbQ ) \, .
\label{eq:inner-prod-procrustes}
\end{equation}
The following lemma states the well-known solution to this problem, which is due to \textcite{Schonemann1966}.

\begin{lemma}[\textcite{Schonemann1966}]
\label{lemma:procrustes-derivation}
Let $\mbU \mbS \mbV^\top$ denote the singular value decomposition of $\mbX^\top \mbY$.
Then ${\mbQ^* = \mbU \mbV^\top}$. Furthermore,
\begin{equation}
\big \langle \mbX, \mbY \mbQ^* \big \rangle = \Vert \mbX^\top \mbY \Vert_* = \sum_i \sigma_i
\label{eq:nuclear-norm-def}
\end{equation}
where $\Vert \cdot \Vert_*$ denotes the nuclear matrix norm and $\sigma_1 \geq \sigma_2 \geq \hdots \geq \sigma_n \geq 0$ are the singular values of $\mbX^\top \mbY$.
\begin{proof}
Let $\mbZ = \mbV^\top \mbQ \mbU$, and note that $\mbZ$ is orthogonal because orthogonal transformations are closed under composition.
The cyclic property of the trace operator implies,
\begin{equation}
\max_{\mbQ \in \cQ} ~ \big \langle \mbX, \mbY \mbQ \big \rangle = \, \max_{\mbQ \in \cQ} ~ \Tr[\mbX^\top \mbY \mbQ ] = \max_{\mbQ \in \cQ} ~ \Tr[\mbS \mbV^\top \mbQ \mbU] = \max_{\mbZ \in \cQ} ~ \Tr[\mbS \mbZ] = \max_{\mbZ \in \cQ} ~ \sum_{i=1}^n \sigma_i z_{ii}
\label{eq:procrustes-derivation}
\end{equation}
where $\{z_{ii}\}_{i=1}^n$ are the diagonal elements of $\mbZ$.
Since $\mbZ$ is orthogonal, we must have $z_{ii} \leq 1$ for all $i \in \{1, \hdots, n\}$.
Since the singular values are nonnegative, the maximum is obtained when each $z_{ii} = 1$.
That is, at optimality we have $\mbZ = \mbV^\top \mbQ^* \mbU = \mbI$, which implies $\mbQ^* = \mbV \mbU^\top$.
Plugging $z_{ii} = 1$ into the final expression of \cref{eq:procrustes-derivation} shows that the optimal objective is given by the sum of the singular values (i.e. the nuclear norm of $\mbX^\top \mbY$).
\end{proof}
\end{lemma}

\subsection{Canonical Correlation Analysis (CCA)}
\label{subsec:cca-additional-notes}

CCA identifies matrices $\mbW_x \in \reals^{n \times n}$ and $\mbW_y \in \reals^{n \times n}$ which maximize the correlation between $\mbX \mbW_x$ and $\mbY \mbW_y$.
Formally, this corresponds to the optimization problem:
\begin{equation}
\begin{aligned}
&\underset{\mbW_x, \mbW_y}{\text{maximize}} & & \Tr[ \mbW_x^\top \mbX^\top \mbY \mbW_y ] \\
&\text{subject to} & & \mbW_x^\top \mbX^\top \mbX \mbW_x = \mbW_y^\top \mbY^\top \mbY \mbW_y = \mbI \, .
\end{aligned}
\label{eq:cca-classic}
\end{equation}
The maximized objective function, $\langle \mbX \mbW_x, \mbY \mbW_y \rangle = \Tr[ \mbW_x^\top \mbX^\top \mbY \mbW_y ]$, generalizes the dot product between two vectors to the Frobenius inner product between $\mbX \mbW_x$ and $\mbY \mbW_y$.
The constraints of the optimization problem constrain the magnitude of the solution---without these constraints, the objective function could be infinitely large, since multiplying $\mbW_x$ or $\mbW_y$ by a real number larger than one proportionally increases $\langle \mbX \mbW_x, \mbY \mbW_y \rangle$.
Intuitively, the typical (Pearson) correlation is equal to the normalized inner product of two vectors, and CCA generalizes this to matrix-valued datasets.

CCA can be transformed into the Procrustes problem by a change of variables.
Assuming that $\mbX^\top \mbX$ and $\mbY^\top \mbY$ are full rank, define $\mbH_x = (\mbX^\top \mbX)^{1/2} \mbW_x$ and $\mbH_y = (\mbY^\top \mbY)^{1/2} \mbW_y$.
Then, (\ref{eq:cca-classic}) can be reformulated as:
\begin{equation}
\begin{aligned}
&\underset{\mbH_x, \mbH_y}{\text{maximize}} & & \Tr \big [ \mbH_x^\top (\mbX^\top \mbX)^{-1/2} \mbX^\top \mbY (\mbY^\top \mbY)^{-1/2} \mbH_y \big ] \\
&\text{subject to} & & \mbH_x^\top \mbH_x = \mbH_y^\top \mbH_y = \mbI \, .
\end{aligned}
\label{eq:cca-change-of-vars}
\end{equation}
By this change of variables, we simplified the constraints of the problem so that $\mbH_x$ and $\mbH_y$ are constrained to be orthogonal matrices.
By applying the cyclic property of the trace operator, and defining $\mbQ = \mbH_y \mbH_x^\top$, $\mbX^\phi = \mbX (\mbX^\top \mbX)^{-1/2}$, $\mbY^\phi = \mbY (\mbY^\top \mbY)^{-1/2}$, we can simplify the problem further:
\begin{equation}
\underset{\mbQ \in \cO}{\text{maximize}} \quad \Tr [ (\mbX^\phi)^\top \mbY^\phi \mbQ] \, .
\label{eq:cca-procrustes-interp}
\end{equation}
Thus, we see that CCA is equivalent to solving the Procrustes problem on $\mbX^\phi$ and $\mbY^\phi$.
Note that $(\mbX^\phi)^\top \mbX^\phi = (\mbY^\phi)^\top \mbY^\phi = \mbI$, and so this change of variables can be interpreted as a whitening operation \cite{Kessy2018}.

From \Cref{lemma:procrustes-derivation}, we see that the optimal objective value to (\ref{eq:cca-procrustes-interp}) is given by the sum of the singular values of $(\mbX^\phi)^\top \mbY^\phi$.
These singular values, which we denote here as $1 \geq \sigma_1 \geq ... \geq \sigma_n \geq 0$, are called \textit{canonical correlation coefficients}.
They are bounded above by one since the singular values of $\mbX^\phi$ and $\mbY^\phi$ are all equal to one, due to the whitening step, and the operator norm\footnote{The operator norm of a matrix $\mbM$, denoted $\Vert \mbM \Vert_{\textrm{op}}$, is equal to the largest singular value of $\mbM$.} is sub-multiplicative:
\begin{equation}
\Vert (\mbX^\phi)^\top \mbY^\phi \Vert_{\textrm{op}} \leq \Vert \mbX^\phi \Vert_{\textrm{op}}  \Vert \mbY^\phi \Vert_{\textrm{op}} = 1 \, .
\end{equation}
Putting these pieces together, we see:
\begin{equation}
\min_{\mbQ \in \cO} ~ \arccos \frac{\langle \mbX^\phi, \mbY^\phi \mbQ \rangle}{\Vert \mbX^\phi \Vert \Vert \mbY^\phi \Vert} = \arccos ~ \frac{\Vert (\mbX^\phi)^\top \mbY^\phi \Vert_* }{\sqrt{n} \cdot \sqrt{n}} = \arccos \, \bigg ( \tfrac{1}{n} \sum_{i=1}^n \sigma_i \bigg )
\end{equation}
which coincides with equation 10 in the main text, since $\sigma_i = \rho_i / n$ for the case of CCA.
Proposition 2 implies that this defines a metric since $\mbX^\phi / \Vert \mbX^\phi \Vert$ and $\mbY^\phi / \Vert \mbY^\phi \Vert$ are matrices with unit Frobenius norm, and because the set of orthogonal transformations is a group of isometries, as established in \cref{subsec:the-orthogonal-group}.

\subsection{Ridge CCA}
\label{subsec:ridge-cca}

Next, we consider metrics based on regularized CCA, which essentially interpolate between the orthogonally invariant metrics discussed in \cref{subsec:orthogonal-procrustes-additional-notes}, and the linearly invariant metrics discussed in \cref{subsec:cca-additional-notes}.
This interpolation is accomplished by specifying a hyperparameter $0 \leq \alpha \leq 1$, where $\alpha=0$ corresponds to unregularized CCA and $\alpha=1$ corresponds to Procrustes alignment (i.e. fully regularized).
We formulate this family of optimization problems as:
\begin{equation}
\begin{aligned}
&\underset{\mbW_x, \mbW_y}{\text{maximize}} & & \Tr[ \mbW_x^\top \mbX^\top \mbY \mbW_y ] \\
&\text{subject to} & & \mbW_x^\top ((1 - \alpha) \mbX^\top \mbX + \alpha \mbI) \mbW_x = \mbW_y^\top ((1 - \alpha) \mbY^\top \mbY + \alpha \mbI)  \mbW_y = \mbI \, .
\end{aligned}
\label{eq:ridge-cca-optimization}
\end{equation}
Notice that when $\alpha=1$, the constraints reduce to $\mbW_x$ and $\mbW_y$ being orthogonal, and thus the objective function can be viewed as maximizing $\langle \mbX, \mbY \mbQ \rangle$ over orthogonal matrices $\mbQ = \mbW_y \mbW_x^\top$.
Thus, we recover Procrustes alignment in the limit of $\alpha = 1$.
Clearly, when $\alpha=0$, \cref{eq:ridge-cca-optimization} reduces to the usual formulation of CCA (see \cref{eq:cca-classic}).

We can solve \cref{eq:ridge-cca-optimization} by following essentially the same procedure outlined in \cref{subsec:cca-additional-notes}, in which we reduce the problem to Procrustes alignment by a change of variables.
In this case, the change of variables corresponds to a partial whitening transformation:
\begin{equation}
\mbH_x = ((1 - \alpha)(\mbX^\top \mbX) + \alpha \mbI))^{1/2}
\quad
\text{and}
\quad
\mbH_y = ((1 - \alpha)(\mbY^\top \mbY) + \alpha \mbI))^{1/2} \, .
\end{equation}
Then, reformulate the optimization problem as:
\begin{equation}
\begin{aligned}
&\underset{\mbH_x, \mbH_y}{\text{maximize}} & & \Tr[ \mbH_x^\top ((1 - \alpha)(\mbX^\top \mbX) + \alpha \mbI))^{-1/2} \mbX^\top \mbY ((1 - \alpha)(\mbY^\top \mbY) + \alpha \mbI))^{-1/2} \mbH_y ] \\
&\text{subject to} & & \mbH_x^\top \mbH_x = \mbH_y^\top \mbH_y = \mbI \, .
\end{aligned}
\end{equation}
Let $\mbQ = \mbH_y \mbH_x$, and let
\begin{equation}
\mbX^\phi = \mbX ((1 - \alpha)(\mbX^\top \mbX) + \alpha \mbI))^{-1/2}
\quad
\text{and}
\quad
\mbY^\phi = \mbY ((1 - \alpha)(\mbY^\top \mbY) + \alpha \mbI))^{-1/2} \, .
\end{equation}
Then, by Proposition 2 and \cref{lemma:procrustes-derivation}, we have the following metric:
\begin{equation}
\min_{\mbQ \in \cO} ~ \arccos \frac{\langle \mbX^\phi, \mbY^\phi \mbQ \rangle}{\Vert \mbX^\phi \Vert \Vert \mbY^\phi \Vert} = \arccos ~ \left \Vert \bigg ( \frac{\mbX^\phi}{\Vert \mbX^\phi \Vert} \bigg )^\top \bigg ( \frac{\mbY^\phi}{\Vert \mbY^\phi \Vert} \bigg ) \right \Vert_* = \arccos \, \bigg ( \sum_{i=1}^n \rho_i \bigg ) \,.
\end{equation}

\subsection{Nonlinear Alignments and Kernel CCA}
\label{subsec:nonlinear-kernel-cca}

We can also consider metrics based on \textit{kernel CCA} \parencite{Hardoon2004}, which generalizes CCA to account for nonlinear alignments.
As its name suggests, this approach belongs to a more general class of \textit{kernel methods} that operate implicitly in high-dimensional (even infinite-dimensional) feature spaces through inner product evaluations.
For a broader review of kernel methods in machine learning, see \parencite{Hofmann2008}.

First, we recall the inner product between two matrices in a finite dimensional feature space $\reals^{m \times p}$:
\begin{equation}
\langle \mbX^\phi, \mbY^\phi \rangle = \Tr [ (\mbX^\phi)^\top \mbY^\phi ] = \sum_{i=1}^m (\mbx^\phi_i )^\top (\mby^\phi_i) \, .
\label{eq:innerprod-matrix-case}
\end{equation}
Here we have introduced notation $\mbx_i^\phi$ and $\mby_i^\phi$ to denote the $p$-dimensional vectors holding features to the $i$\textsuperscript{th} network input.
In kernel CCA, we consider more general feature mappings $\mbx_i \mapsto x^\phi_i$ and $\mby_i \mapsto y^\phi_i$, where each $x^\phi_i \in \cH$ and $y^\phi_i \in \cH$ are vectors in some Reproducing Kernel Hilbert Space (RKHS).
That is, instead of having two matrices $\mbX^\phi$ and $\mbY^\phi$ to represent the network representations in the feature space, we instead consider the collections of vectors: $X^\phi = \{x_1^\phi, \hdots, x_m^\phi\}$ and $Y^\phi = \{y_1^\phi, \hdots, y_m^\phi\}$.

Given a choice of a positive-definite kernel function $k$, we begin by computing two $m \times m$ un-centered kernel matrices:
\begin{equation}
[\widetilde{\mbK}_x]_{ij} = k(\mbx_i, \mbx_j) = \langle x_i^\phi, x_j^\phi \rangle
\quad \text{and} \quad
[\widetilde{\mbK}_y]_{ij} = k(\mby_i, \mby_j) = \langle y_i^\phi, y_j^\phi \rangle
\end{equation}
for $i, j \in \{1, \hdots, m\}$.
Then, we define the centered kernel matrices: $\mbK_x = \mbC \widetilde{\mbK}_x \mbC$ and $\mbK_y = \mbC \widetilde{\mbK}_y \mbC$, where $\mbC = \mbI - \tfrac{1}{m} \boldsymbol{1} \boldsymbol{1}^\top$ is the centering matrix.

The classic form of CCA (\ref{eq:cca-classic}) can then be reformulated terms of purely kernel operations \cite{Hardoon2004,Uurtio2017}:
\begin{equation}
\label{eq:nonlinear-kernel-cca}
\begin{aligned}
& \underset{\mbW_x, \mbW_y}{\text{maximize}}
& & \Tr \left [ \mbW_x^\top \mbK_x \mbK_y \mbW_y \right ] \\
& \text{subject to}
& & \mbW^\top_x \mbK_x^2 \mbW_x = \mbW^\top_y \mbK_y^2 \mbW_y = \mbI \, .
\end{aligned}
\end{equation}
One can show that this optimization problem is equivalent (up to a change of variables) from the classic CCA problem when a linear kernel function, $k(\mbx_i, \mbx_j) = \mbx_i^\top \mbx_j$, is used.
Furthermore, one can generalize the regularization scheme for CCA (see \cref{subsec:ridge-cca}), 
\begin{equation}
\begin{aligned}
& \underset{\mbW_x, \mbW_y}{\text{maximize}}
& & \Tr \left [ \mbW_x^\top \mbK_x \mbK_y \mbW_y \right ] \\
& \text{subject to}
& & \mbW^\top_x ( (1-\alpha) \mbK_x^2 + \alpha \mbK_x) \mbW_x = \mbW^\top_y ( (1-\alpha) \mbK_y^2 + \alpha \mbK_y) \mbW_y = \mbI \, .
\end{aligned}
\end{equation}


\subsection{Geodesic Distances on Kendall's Shape Space}

We now consider a modification of the Procrustes alignment problem, where we optimize over the special orthogonal group (i.e. the set of orthogonal matrices with $\textrm{det}(\mbQ) = +1$)
\begin{equation}
\mbR^* = \argmin_{\mbR \in \cS\cO} ~ \Vert \mbX - \mbY \mbR \Vert = \argmax_{\mbR \in \cS\cO} ~ \Tr[\mbX^\top \mbY \mbR] \, .
\end{equation}
We can obtain the solution by a minor modification of \Cref{lemma:procrustes-derivation}.
We let $\mbX^\top \mbY = \tilde{\mbU} \tilde{\mbS} \tilde{\mbV}^\top$ denote the ``optimally signed'' singular value decomposition of $\mbX^\top \mbY$ in which $\tilde{\mbU} \in \cS\cO$, $\tilde{\mbV} \in \cS\cO$, and $\tilde{\mbS}$ is a diagonal matrix of signed singular values: $\tilde{\sigma}_1 \geq \hdots \geq \tilde{\sigma}_{n-1} \geq |\tilde{\sigma}_n| \geq 0$.
Thus, all optimally signed singular values are positive except if $\textrm{det}(\mbX^\top \mbY) < 0$, in which case the final singular value is negated, $\tilde{\sigma}_n = -\sigma_n$, so that $\det(\tilde{\mbU}) = \det(\tilde{\mbV}) = +1$.
Then the optimal rotation is given by $\mbR^* = \tilde{\mbV} \tilde{\mbU}^\top$.
See \textcite{Le1991} for a proof.

We refer the reader to Chapters 4 and 5 of \textcite{dryden2016} for further details.
When $\cG = \cS\cO$, our Proposition 1 corresponds to Riemannian distance in size-and-shape space (sec. 5.3, \cite{dryden2016}).
Likewise, $\cG = \cS\cO$, our Proposition 2 corresponds to Riemannian distance Kendall's shape space (sec. 4.1.4, \cite{dryden2016}).

\subsection{Centered Kernel Alignment (CKA) and Representational Similarity Analysis (RSA)}

Linear CKA \cite{Kornblith2019} and RSA \cite{Kriegeskorte2008} are two closely related methods that, in essence, evaluate the similarity between $\mbX \mbX^\top$ and $\mbY \mbY^\top$ to capture the similarity of neural representations.
When the data are mean-centered as a preprocessing step, these are $m \times m$ covariance matrices capturing the correlations in neural activations over the $m$ test images.
Several variants of RSA exist.
For example, one can compute the pairwise Euclidean distances between all $m$ hidden layer activation patterns, resulting in representational distance matrices (RDMs) instead of the covariance matrices mentioned above.
Likewise, nonlinear extensions of CKA use nonlinear kernel functions to compute centered kernel matrices $\mbK_x$ and $\mbK_y$, as defined above in \cref{subsec:nonlinear-kernel-cca}.
When a linear kernel function is used (i.e. in linear CKA), the centered kernel matrices reduce to the usual covariance matrices $\mbK_x = \mbX \mbX^\top$ and $\mbK_y = \mbY \mbY^\top$.

In essence, these methods proceed by computing the similarity between $\mbK_x$ and $\mbK_y$.
\textcite{Kriegeskorte2008} proposed taking the Spearman correlation between the upper-triangular entries of these matrices.
This measure of similarity does not produce a metric, as we verified empirically in the main text.
\textcite{Kornblith2019} proposed to use the following quantity (assuming centered kernels):
\begin{equation}
\text{CKA}(\mbK_x, \mbK_y) = \frac{\Tr [ \mbK_x \mbK_y ]}{\sqrt{\Tr [\mbK_x^2] \cdot \Tr [\mbK_y^2] }}
\end{equation}
which is known as centered kernel alignment (originally defined in \cite{Cristianini2002,Cortes2012}).

While CKA as originally formulated does not produce a metric, we can modify it to satisfy the requirements of a metric space.
First, note that:
\begin{equation}
\text{CKA}(\mbK_x, \mbK_y) = \cos \big [ d_\theta(\mbK_x, \mbK_y) \big ]
\end{equation}
where $d_\theta$ is the angular distance (see \cref{eq:appendix_a_ang_distance}) over $\reals^{m \times m}$ matrices.
Thus, one can apply $\arccos(\cdot)$ to CKA achieve a proper metric.
For example, a metric based on linear CKA can be calculated as follows:
\begin{equation}
\label{eq:cka-metric-angular-distance}
d_\theta(\mbX \mbX^\top, \mbY \mbY^\top) = \arccos \left [ \frac{\Vert \mbX^\top \mbY \Vert^2}{\Vert \mbX \mbX^\top \Vert \Vert \mbY \mbY^\top \Vert} \right ]
\end{equation}
where, as before, all norms denote the Frobenius matrix norm.
Note that this calculation bears some similarity to the fully regularized CCA distance:
\begin{equation}
\theta_1(\mbX, \mbY) = \min_{\mbQ \in \cO} \arccos \left [ \frac{\langle \mbX, \mbY \mbQ \rangle }{\Vert \mbX \Vert \cdot \Vert \mbY \Vert} \right ] = \arccos \left [ \frac{\Vert \mbX^\top \mbY \Vert_*}{\Vert \mbX \Vert \Vert \mbY \Vert} \right ]
\end{equation}
The two differences between these metrics are that (a) CKA uses the squared Frobenius norm instead of the nuclear norm to measure the scale of $\mbX^\top \mbY$ in the numerator, and (b) CKA normalizes by the norms of the covariances, $\mbX \mbX^\top$ and $\mbY \mbY^\top$, rather than the norms of the matrices themselves.

While this manuscript was undergoing review, \textcite{Shahbazi2021} published a different modification of CKA and RSA to satisfy the properties of a metric space.
They advocate using the Riemannian metric over positive-definite matrices:
\begin{equation}
d(\mbK_x, \mbK_y) = \sqrt{\sum_{i=1}^m \log^2 (\lambda_i) } \,,
\end{equation}
where $\lambda_1, \hdots, \lambda_m$ are the eigenvalues of $\mbK_x^{-1} \mbK_y$.
This calculation is appealing because it exploits the fact that $\mbK_x$ and $\mbK_y$ are positive-definite matrices by construction.
The extension of CKA discussed above utilizes the generic angular distance between $m \times m$ matrices, which are not necessarily positive-definite.

\section{Probabilistic interpretations of generalized shape metrics}
\label{appendix:d_stochastic_layers}

To extend generalized shape metrics to stochastic neural representations, we must introduce some additional notation and formalize network representations as random variables (rather than $m \times n$ matrices).
We can model neural representations as independent random variables when conditioned on the input.
That is, let $X$ and $Y$ denote random variables on $\reals^n$, which correspond to $n$-dimensional neural responses to a stochastic input.\footnote{As in the main text, we can define feature maps $X \mapsto X^\phi$ and $Y \mapsto Y^\phi$ which establish a common dimensionality between networks of dissimilar sizes.}
Further, let $Z$ be some random variable corresponding to process of sampling an input to the network (e.g. choosing one of $m$ input images at random).
Then, the joint distribution over representations and inputs decomposes as $P(X, Y, Z) = P(X \mid Z) P(Y \mid Z) P(Z)$ for any pair of networks $X$ and $Y$.

The goal of this section is to define functions $d(X, Y)$ that are metrics over the set of random variables with outcomes on $\reals^n$, and which are natural extensions of Proposition 1 and 2 in the main text.
The key step towards achieving this goal is to establish a Hilbert space for random vectors.
We provide a short and informal demonstration of this below, but refer the reader to Chapter 2 of \textcite{Tsiatis2007} for a more complete treatment.

First, we establish that the set of random vectors is a vector space.
The zero vector corresponds to a random vector that is equal to the zero vector on $\reals^n$ almost surely.
Vector addition $X + Y$ creates a new random vector from two inputs $X$ and $Y$.
Intuitively, we can draw samples from $X + Y$ by first sampling $X$ and $Y$ and then adding their outcomes.
Scalar multiplication $\alpha X$ creates a new random vector given the input $X$ and a scalar $\alpha \in \reals$.
Intuitively, we can sample $\alpha X$ by first drawing a sample from $X$ and multiplying this outcome by $\alpha$.
We can then define the inner product between two random vectors in the following lemma.

\begin{lemma*}
\label{lemma:inner-product-of-rand-vars}
Let $X$ and $Y$ be random vectors associated with some joint probability density function $p(\mbx, \mby)$ for all $\mbx \in \reals^n$ and $\mby \in \reals^n$.
Then,
\begin{equation}
\langle X, Y \rangle = \expect [ \mbx^\top \mby ] \, ,
\end{equation}
is an inner product over the set of random vectors, where the expectation is taken over joint samples of $X$ and $Y$.
\end{lemma*}
\begin{proof}
Using the linearity of expectation and the inner product on $\reals^n$, it is easy to prove that the inner product is symmetric,
\begin{equation}
\langle X, Y \rangle = \expect[ \mbx^\top \mby ] = \expect [\mby^\top \mbx] = \langle Y, X \rangle \, ,
\end{equation}
and linear,
\begin{equation}
\langle M + \alpha X, Y \rangle = \expect [ (\mbz + \alpha \mbx)^\top \mby ] = \expect [ \mbz^\top \mby ] + \alpha \expect [ \mbx^\top \mby ] = \langle M, Y \rangle  + \alpha \langle X, Y \rangle
\end{equation}
for any random vector $M$ and $\alpha \in \reals$.
All that remains is to prove is that $\langle \cdot, \cdot \rangle$ is positive definite, we first note that the mapping $\mbx \mapsto \mbx^\top \mbx$ is a convex function of $\mbx$.
Then, we apply Jensen's inequality and the positive definiteness of the inner product on $\reals^n$ to show:
\begin{equation}
\langle X, X \rangle = \expect [ \mbx^\top \mbx ] \geq (\expect \mbx)^\top (\expect \mbx) \geq 0 \,.
\end{equation}
Further $\expect [\mbx^\top \mbx] = 0$ only when $\mbx = \boldsymbol{0}$, almost surely.
Thus, $\langle X, X \rangle = 0$ if and only if $X = 0$.
\end{proof}

To begin, we consider a special case where the neural responses are deterministic, but the inputs are randomly chosen.
That is, to draw a sample of $(X, Y)$, we first sample an input $\mbz \sim P(Z)$ and then calculate $\mbx = f_x(\mbz)$ and $\mby = f_y(\mbz)$, where $f_x$ and $f_y$ are functions mapping the input space to $\reals^n$.

In the simplest case, $P(Z)$ is a uniform distribution over a discrete set of $m$ network inputs. In this case, we can compute the required inner products exactly.
Let $\mbz_i$ denote the $i$\textsuperscript{th} input to the networks, and let $\mbX \in \reals^{m \times n}$ and $\mbY \in \reals^{m \times n}$ denote matrices that stack the neural responses, $f_x(\mbz_i)$ and $f_y(\mbz_i)$ row-wise.
Then we have
\begin{equation}
\langle X, Y \rangle = \expect [ \mbx^\top \mby ] = \frac{1}{m} \sum_{i=1}^m f_x(\mbz_i)^\top f_y(\mbz_i) = \frac{1}{m} \langle \mbX, \mbY \rangle \, ,
\label{eq:rand-var-inner-product-discrete-case}
\end{equation}
where the final inner product $\langle \mbX, \mbY \rangle = \Tr[\mbX^\top \mbY]$ is the typical Frobenius inner product between matrices that we have used throughout.
Because these inner products coincide up to a uniform scaling factor, we can reinterpret the metrics defined in the main text (Propositions 1 \& 2) as providing a notion of distance between deterministic neural responses that are drawn uniformly from a set of $m$ inputs.

In many cases, the number of possible inputs to a network is effectively infinite, so we can consider $P(Z)$ to be a continuous distribution.
In this scenario, the inner product becomes:
\begin{equation}
\langle X, Y \rangle = \int p(\mbz) f_x(\mbz)^\top f_y(\mbz) \, \textrm{d}\mbz
\end{equation}
which is generally intractable to compute.
For example, we typically do not know how to evaluate the density $p(\mbz)$.
This is the case, for example, when $P(Z)$ corresponds to the distribution over all ``natural images.''
If we are given independent samples $\mbz_i \sim P(Z)$, for $i = 1, \hdots, m$, then the integral can be approximated as
\begin{equation}
\int p(\mbz) f_x(\mbz)^\top f_y(\mbz) \, \textrm{d}\mbz \approx \frac{1}{m} \sum_{i=1}^m f_x(\mbz_i)^\top f_y(\mbz_i) = \frac{1}{m} \langle \mbX, \mbY \rangle  \, ,
\label{eq:rand-var-inner-product-continuous-case}
\end{equation}
which coincides with (\ref{eq:rand-var-inner-product-discrete-case}).
Thus, we can also interpret generalized shape metrics (Propositions 1 \& 2) as being approximations to metrics that capture representational dissimilarity over a continuous distribution of input patterns.
This final interpretation is appealing from both scientific and engineering perspectives.
In neuroscience, we expect animals to encounter sensory input patterns probabilistically from an effectively infinite range of possibilities.
Likewise, in machine learning, we are interested in how deep artificial networks generalize to ``real-world'' applications.
In short, the space of possible future inputs is generally more numerous than the space of inputs used for training and validation.
Nonetheless, if the statistics of the test set match the ``real world,'' then (\ref{eq:rand-var-inner-product-continuous-case}) tells us that we can approximate the ``true'' distance between network representations appropriately.

The results shown in Figures 3B and 3C in the main text can now be properly interpreted as varying the choice of $m$ (sample size) in the approximation of the integral appearing in equation (\ref{eq:rand-var-inner-product-continuous-case}).

The framework above can also be readily extended to define metrics between stochastic neural representations, which are ubiquitous in both biology (due to ``noise'') and machine learning (e.g. dropout layers).
We view this as an intriguing direction for future research that is enabled by our theoretical framing of neural representations.

\section{Experimental Methods}
\label{appendix:e_experimental_methods}

Code accompanying this paper can be found at --- \url{https://github.com/ahwillia/netrep}

\subsection{Experiments on sample size (Fig. 3)}

We ran all experiments on a pair of convolutional neural networks trained on CIFAR-10.
The architecture is shown in Table~\ref{tab:architecture}.
In Figure 3A, we sampled activations from the three layers following the stride-2 convolutions.
We did a brute-force search over circular shifts along the width and height dimensions.
When comparing two layers with unequal dimensions, we upsampled the layer with smaller width and height by linear interpolation.
The remaining panels in Figure 3 were computed using activations from the final layer before average pooling.

\begin{table}[h]
    \centering
    \begin{tabular}{p{5cm}}
        \toprule
        $3\times 3$ conv. 64-BN-ReLU\\
        $3\times 3$ conv. 64-BN-ReLU\\
        $3\times 3$ conv. 64-BN-ReLU\\
        $3\times 3$ conv. 64 stride 2-BN-ReLU\\
        $3\times 3$ conv. 128-BN-ReLU\\
        $3\times 3$ conv. 128-BN-ReLU\\
        $3\times 3$ conv. 128-BN-ReLU\\
        $3\times 3$ conv. 128 stride 2-BN-ReLU\\
        $3\times 3$ conv. 256-BN-ReLU\\
        $3\times 3$ conv. 256-BN-ReLU\\
        $3\times 3$ conv. 256-BN-ReLU\\
        $3\times 3$ conv. 256 stride 2-BN-ReLU\\
        Global average pooling\\
        Logits\\
        \bottomrule
    \end{tabular}
    \caption{The architecture used for experiments in Fig. 3. All convolutions use zero padding to maintain the size of the feature map.}
    \label{tab:architecture}
\end{table}

\subsection{Allen Brain Observatory}

Data were accessed through the Allen Software Development Kit (AllenSDK --- \url{https://allensdk.readthedocs.io/en/latest/}).
All isolated single units that met the default quality control standards were loaded and pooled across sessions.
The anatomical location of each unit in Common Coordinate Framework (CCF; \cite{Wang2020}) was extracted and categorized into anatomical regions according to the reference atlas, using the finest scale anatomical parcellation.
Spike counts were calculated over \texttt{0.033355} ms timebins (duration of a single movie frame), over 1600 frames.
Spikes were then smoothed with a Gaussian filter with a standard deviation of 20 bins (frames), and averaged over 10 trials (repeats of the movie).
Then, we projected the data onto the top 100 principal components, resulting in a matrix $\mbX_k \in \reals^{1600 \times 100}$ for each brain region $k = \{1, \hdots, K\}$.
Regions with fewer than 100 neurons across all sessions were excluded.
The following set of 48 regions, listed by their standard abbreviations, contained more than 100 neurons and were then studied for further analysis: \texttt{APN, AUDd5, AUDpo5, AUDpo6a, CA1, CA3, DG-mo, DG-sg, Eth, LGd-co, LGd-ip, LGd-sh, LGv, LP, MB, MGd, MGv, PO, POL, ProS, SGN, SSp-bfd2/3, SSp-bfd4, SSp-bfd5, SUB, TEa5, TH, VISa2/3, VISa4, VISa5, VISa6a, VISal2/3, VISal4, VISal5, VISam2/3, VISam4, VISam5, VISam6a, VISp2/3, VISp4, VISp5, VISp6a, VPM, alv, ccs, dhc, fp, or}.

Dendrograms were computed and visualized using tools available in the scipy library \cite{scipy}.
We used Ward's linkage criterion to compute the hierarchical clusterings.

We performed kernel ridge regression to predict anatomical hierarchy scores (defined in \cite{Harris2019}) 29 regions: \texttt{AUDd5, AUDpo5, AUDpo6a, LGd-co, LGd-ip, LGd-sh, LP, MGd, MGv, POL, SSp-bfd2/3, SSp-bfd4, SSp-bfd5, TEa5, VISa2/3, VISa4, VISa5, VISa6a, VISal2/3, VISal4, VISal5, VISam2/3, VISam4, VISam5, VISam6a, VISp2/3, VISp4, VISp5, VISp6a}.
Two regions, \texttt{PO} and \texttt{VPM}, were excluded from the analysis as they were outliers with exceptionally high and low hierarchy scores.
The other regions were excluded because they either had undefined hierarchy scores or had fewer than 100 neurons.
We used the scikit-learn implementation of kernel ridge regression, \texttt{KernelRidge(alpha=0.01, gamma=1.0, kernel="rbf")}, and fit the model 100 separate times on different approximate Euclidean embeddings found by multi-dimensional scaling (MDS).
The error bars in Fig. 5B show range of estimates from different MDS embeddings.
An embedding dimension of $L=20$ was used in all cases.

\begin{figure}[htbp]
\centering
\includegraphics[width=\linewidth]{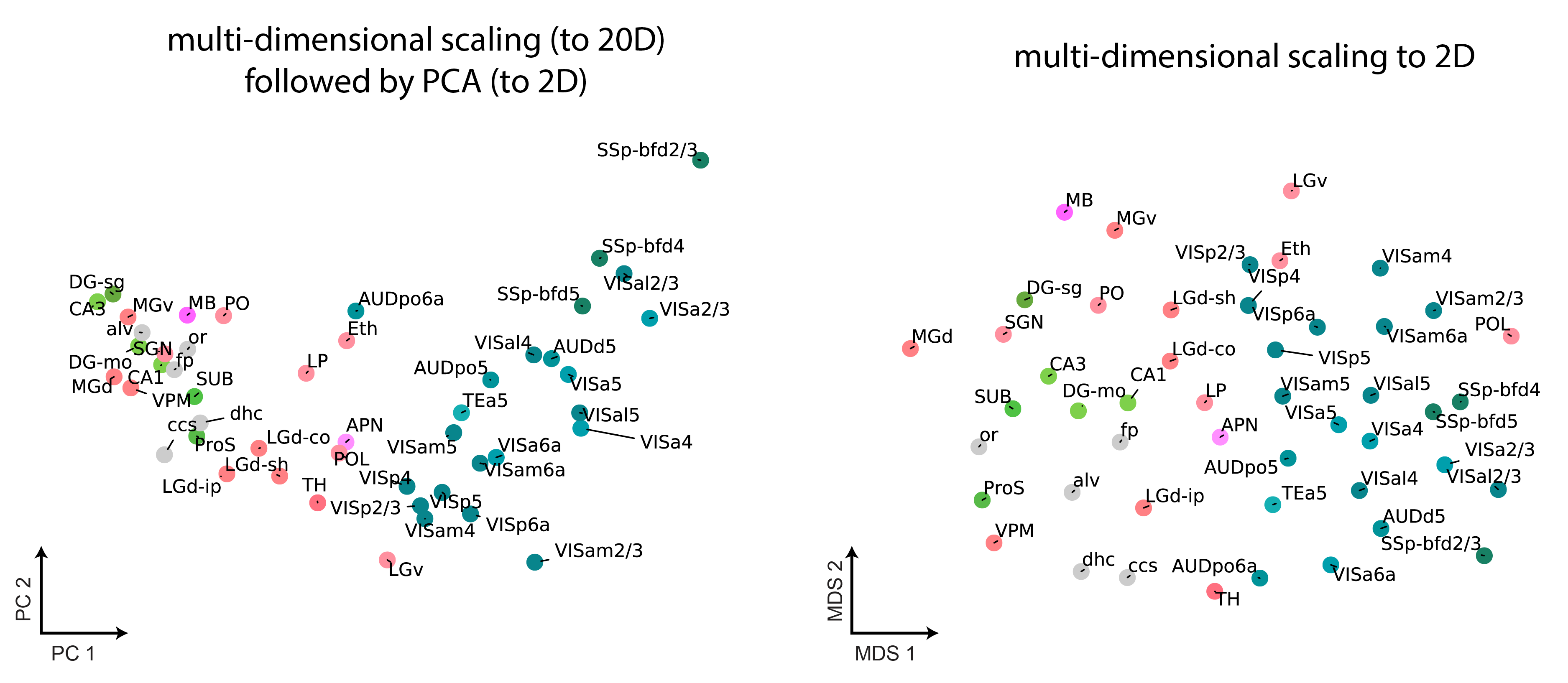}
\caption{
Performing MDS directly to $L=2$ dimensions (right) produces a distinct low-dimensional visualization of the ABO dataset from multi-dimensional scaling to $L=20$ dimensions, following by PCA projection down to 2D (left).
As shown in the main text (Fig. 4C), the MDS embedding to $L=20$ dimensions produces a dramatically better approximation of the true metric space than the embedding to $L=2$ dimensions.
Thus, we advocate using the former over the latter for downstream modeling tasks.
}
\label{fig:mds_supplement}
\end{figure}

If our goal is only to visualize the data in 2D we may apply MDS with an embedding dimension of $L=2$.
How does this embedding differ from a larger embedding of $L=20$?
\Cref{fig:mds_supplement} demonstrates that qualitatively distinct structures emerge from these two procedures.

\subsection{NAS-Bench-101}
We obtained checkpoints for 2000 randomly-selected NAS-Bench-101 architectures trained for 108 epochs following the protocol described in~\cite{Ying19} and computed the similarity between activations of every possible pair of these architectures on the CIFAR-10 test set, using an Apache Beam pipeline operating on offsite hardware. In total, the computational cost of these experiments was 260 core-years, including pilot experiments and several experiments not included in the paper.

For ridge regression analyses in Fig. 5D, we train on 80\% of the data, use 10\% of the data as a validation set to select the optimal ridge hyperparameter and the kernel bandwidth, and compute $R^2$ on remaining 10\% of the data. 

In Figure~\ref{nasbench_arch}, we show the skeleton of the NAS-Bench-101 architecture along with the layers from which we extract representations.
\begin{figure}[htbp]
\centering
\includegraphics[width=0.2\linewidth]{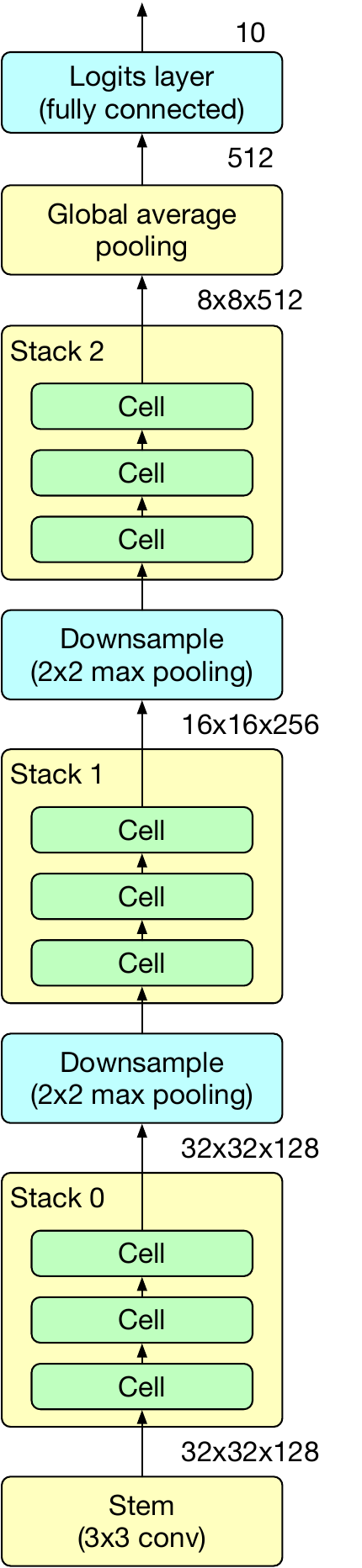}
\caption{
Diagram of the skeleton of the NAS-Bench-101 architecture. The architecture of each cell (shown in green) is selected from a fixed space, described further by \textcite{Ying19}, and all cells within a single architecture are identical except for the number of channels, which differs by stack. In Fig. 5, we show the results we obtain by analyzing the representations of the outputs of the layers shown in yellow.
}
\label{nasbench_arch}
\end{figure}

\end{appendices}

\end{document}